\documentclass{article}

\usepackage[pagebackref=true,breaklinks=true,colorlinks,urlcolor={violet!90!black},linkcolor={violet!90!black},citecolor={blue!90!black},bookmarks=false]{hyperref}

 \usepackage[main, final]{neurips_2026}

\usepackage[utf8]{inputenc} 
\usepackage[T1]{fontenc}    
\usepackage{hyperref}       
\usepackage{url}            
\usepackage{booktabs}       
\usepackage{amsfonts}       
\usepackage{nicefrac}       
\usepackage{microtype}      
\usepackage{xcolor}         

\usepackage{amsmath}
\usepackage{amssymb}
\usepackage{mathtools}
\usepackage{amsthm}
\usepackage{siunitx}
\usepackage{microtype}
\usepackage{setspace}
\usepackage{xspace}
\usepackage{textcomp}
\usepackage{wrapfig}
\usepackage{booktabs} 
\usepackage{multirow}
\usepackage{threeparttable}
\usepackage{tabularx}
\usepackage{siunitx}

\usepackage{graphicx}
\usepackage{subcaption} 
\usepackage{wrapfig}
\usepackage{xcolor}
\usepackage{colortbl}

\usepackage{tikz}
\usepackage{pgfplots}
\usetikzlibrary{intersections}
\usepgfplotslibrary{fillbetween,groupplots,dateplot}
\pgfplotsset{compat=newest, scaled y ticks=false}

\usepackage{algorithm}
\usepackage{algorithmic}

\usepackage{hyperref}
\usepackage{url}
\usepackage{pifont}
\usepackage{amssymb}
\usepackage{bookmark}

\usepackage[capitalize,noabbrev]{cleveref}

\theoremstyle{plain}
\newtheorem{theorem}{Theorem}[section]
\newtheorem{proposition}[theorem]{Proposition}
\newtheorem{lemma}[theorem]{Lemma}
\newtheorem{corollary}[theorem]{Corollary}
\theoremstyle{definition}

\theoremstyle{remark}
\newtheorem{remark}[theorem]{Remark}
\usepackage[textsize=tiny]{todonotes}

\newcommand{\method}{\textsc{GraMO}\xspace}

\newcommand{\stot}{\texttt{S2T}\xspace}
\newcommand{\ttot}{\texttt{T2T}\xspace}
\newcommand{\ttos}{\texttt{T2S}\xspace}

\definecolor{deepblue}{RGB}{0,80,150}
\definecolor{crimsonred}{RGB}{180,30,50}
\definecolor{lightgray}{RGB}{245,245,245}

\definecolor{first}{HTML}{006400}  
\definecolor{second}{HTML}{006EB8} 
\newcommand{\one}[1]{\textcolor{first}{\bf#1}}
\newcommand{\two}[1]{\textcolor{second}{\bf#1}}

\usepackage{etoc}
\usepackage{newtxtext}

\title{Graph Mamba Operator: A Latent Simulator for Interacting Particle Systems}

\author{%
  Karn Tiwari$^{1}$ \quad
  Niladri Dutta$^{2}$ \quad
  Prathosh A P$^{1, 5}$ \quad
  N M Anoop Krishnan$^{3, 4}$ \\[0.5em]
  Indian Institute of Science, Bangalore, India \\
  Indian Institute of Technology Delhi, New Delhi, India\\
  \texttt{\{karntiwari,niladridutta,prathosh\}@iisc.ac.in} \\
  \texttt{krishnan@iitd.ac.in}
}

\begin{document}

\maketitle
\footnotetext[1]{Department of Electrical Communication Engineering, Indian Institute of Science, Bengaluru, India}
\footnotetext[2]{Department of Computer Science and Automation, Indian Institute of Science, Bengaluru, India}
\footnotetext[3]{Yardi School of Artificial Intelligence, Indian Institute of Technology Delhi, New Delhi, India}
\footnotetext[4]{Department of Civil and Environmental Engineering, Indian Institute of Technology Delhi, New Delhi, India}
\footnotetext[5]{LatentForce.ai}

\begin{abstract}
Modeling interacting dynamical systems requires capturing spatial interactions alongside long-range temporal dependencies. Graph neural networks (GNNs) provide a natural representation but typically rely on autoregressive rollouts and treat spatial and temporal dynamics separately, leading to error accumulation over long horizons. Existing approaches also focus on local interactions and short temporal contexts, limiting their ability to capture multi-hop dependencies and global structure. We introduce the Graph Mamba Operator (\method), a latent-space simulator that integrates state-space models with graph-based interaction learning. In contrast to prior work that sequences nodes or applies spatial and temporal updates in separate stages, \method couples graph-based interactions and temporal state updates within a single recurrence. For fixed inputs, the update is affine in the latent state, with input-dependent coefficients that adapt across regimes. We evaluate \method on N-body systems, motion capture, and robotics datasets, achieving the lowest mean squared error across benchmarks.
\end{abstract}

\etocdepthtag.toc{mtmain}

\section{Introduction}
Interacting particle systems, articulated bodies, and deformable structures together describe a wide range of phenomena governing locomotion, manipulation, and control. These systems share a common mathematical foundation: collections of interacting entities whose collective behavior emerges over time from local interaction rules governed by differential equations. Classical numerical integrators often solve these equations at small time steps for stability, which becomes computationally prohibitive when the number of bodies is large or the temporal horizon is long \citep{donnelly2005symplectic,bou2013time,hollingsworth2018molecular}.

Latent simulators offer a tractable alternative. By learning a compact representation of the system state and propagating it forward, they reduce the cost of explicit force computation by propagating a learned latent state while retaining access to trajectory-level predictions. For particle-based systems, the choice of latent representation is constrained by two requirements. First, the representation must reflect the graph structure of interactions, since the dynamics of each body depend on its neighbors through the interaction potential. Graph neural networks (GNNs) are the standard primitive for this \citep{kipf2018neural,sanchez2020learning,pfaff2020learning,bishnoi_learning_2023,bihani2024egraffbench,xu2024equivariant}. Second, the representation must persist over time, since particle trajectories carry long-range dependencies through inertia, recurring contacts, and slow collective modes. Stepwise GNN rollouts handle the first requirement but not the second: each step is conditioned only on the previous state, and prediction errors compound over the rollout horizon \citep{xu2024equivariant}.

State-space models provide a natural primitive for the temporal component. They evolve a hidden state through a linear recurrence with structured initialization \citep{gu2020hippo,gu2023mamba}, support long-range dependencies at linear cost, and admit input-dependent parameters that adapt across regimes \citep{gu2023mamba,dao2024transformers}. Recent work applies SSMs to PDE learning \citep{tiwari2025latent} and to graphs through node sequencing \citep{behrouz2024graph,wang2024graph}, but sequencing discards the topology that governs particle interactions. A latent simulator for particle systems requires the SSM recurrence and the graph operator to act jointly on a single latent state at each step, so that temporal memory and spatial interactions evolve within a single update.

We propose the \emph{Graph Mamba Operator} (\method), a latent simulator that couples graph propagation with state-space evolution in a single recurrence. At each step, the latent state is jointly updated along graph edges and the temporal axis, with input-dependent coefficients that adapt to the current system state. The joint update addresses the coupling requirement identified above, while the input-dependent coefficients accommodate regime variation across contact, slow drift, and fast transitions. Trajectories are produced through latent-state evolution rather than recursive prediction in the observation space, reducing reliance on recursive observation-space prediction within each forecast window; feedback rollouts can still accumulate error. This view also connects the update to time-varying Koopman dynamics, where the latent transition remains linear while the operator adapts to the current graph state \citep{kutz2016koopman}. Our contributions are summarized as follows:

\begin{enumerate}
    \item \textbf{\method.} We present a graph-coupled state-space operator that evolves a persistent latent state under input-dependent dynamics, jointly modeling spatial interactions and temporal memory in a single latent recurrence.
    \item \textbf{Theoretical Analysis.} We provide stability and long-range propagation results for the latent recurrence, together with an adaptive graph moving-average representation.
    \item \textbf{Experimental Validation.} Our method achieves improved trajectory prediction and zero-shot generalization across N-body simulation, motion capture, and robotics benchmarks, outperforming graph-based and neural operator baselines.
\end{enumerate}

\section{Related Work}
\textbf{State-Space Models (SSMs) on Graphs.}
State-space models have shown strong performance in sequence modeling~\citep{gu2023mamba,dao2024transformers} and have recently been applied to PDE learning~\citep{tiwari2025latent}. Several works explore adapting SSMs to graphs~\citep{qu2024survey,wang2024graph,behrouz2024graph}, often by imposing artificial sequential orderings through node sorting or random walks. Such constructions disrupt the intrinsic graph topology and are poorly suited for physical systems with structured interactions. Spatiotemporal graph models incorporating temporal dynamics~\citep{sahili2023spatio,cini2023scalable} further extend this line of work, but do not fully exploit the long-range memory and stability properties of SSMs for spatial correlation. Unlike prior graph-state-space models that impose artificial node orderings, \method couples spatial propagation and temporal memory within a single discrete graph-coupled state-space recurrence.

\textbf{Neural Operator.}
Neural operators learn mappings between function spaces and provide data-driven alternatives to classical numerical solvers~\citep{kovachki2023neural,li2020fourier,li2023fourier,xu2024equivariant}. For dynamical systems, this corresponds to learning an operator that maps an observed trajectory history to future evolution. This view is closely related to latent simulator frameworks, where the system is first encoded into a compact latent state and then propagated forward using learned dynamics. Several lines of work instantiate this idea in different ways. Koopman-inspired methods learn latent representations in which nonlinear dynamics are approximated by linear evolution, enabling efficient prediction and control~\citep{otto2021koopman,brunton2016koopman,li2019learning}. However, many such approaches rely on fixed or weakly adaptive latent operators, limiting their ability to model strongly time-varying systems. Recent state-space operator models address temporal dependencies by maintaining persistent latent memory through structured recurrences~\citep{gu2023mamba,tiwari2025latent}, but they typically do not explicitly encode graph-structured interactions.

\textbf{Graph Neural Networks (GNNs).} Graph neural networks model particles or bodies as nodes and interactions as edges~\citep{kipf2018neural,mrowca2018flexible,sanchez2020learning,pfaff2020learning}. While effective for local physical interactions, standard message-passing models often struggle with multi-hop dependencies and long-term forecasting, especially when used in autoregressive rollouts. More general spatiotemporal graph models~\citep{li2017diffusion,jin2023spatio, xu2022geodiff, xu2023eqmotion} improve temporal modeling, but spatial and temporal evolution are often handled by separate modules. In contrast, \method is a graph-structured latent simulator that couples graph-based interaction learning with state-space evolution inside a single neural operator formulation, allowing spatial propagation and temporal memory to jointly evolve the physical dynamics.

\section{Background and Problem Setting}
\textbf{State-Space Models (SSMs).} A continuous-time linear state-space model describes the evolution of a hidden state in response to external inputs. Let the input be $\mathbf{x}(t) \in \mathbb{R}^{d_x}$, the hidden state be $\mathbf{h}(t) \in \mathbb{R}^{d_h}$, and the output be $\mathbf{y}(t) \in \mathbb{R}^{d_y}$. The dynamics are given as follows:
\begin{align}
    \dot{\mathbf{h}}(t) = \mathbf{A} \mathbf{h}(t) + \mathbf{B} \mathbf{x}(t), 
    \qquad \mathbf{y}(t) = \mathbf{C} \mathbf{h}(t),
\end{align}
where $\mathbf{A}$, $\mathbf{B}$, and $\mathbf{C}$ are learnable system matrices.

\textbf{Discretization.} To integrate SSMs into neural architectures, we discretize the system using a step size $\Delta$ under a zero-order hold (ZOH) assumption as follows:
\begin{align}
    \mathbf{h}[k] &= \bar{\mathbf{A}}\mathbf{h}[k-1] + \bar{\mathbf{B}}\mathbf{x}[k], \qquad
    \mathbf{y}[k] = \mathbf{C}\mathbf{h}[k],
\end{align}
where,
\begin{align}
    \label{eq:ssm_discrete}
    \bar{\mathbf{A}} = e^{\Delta \mathbf{A}}, \quad 
    \bar{\mathbf{B}} = \int_0^\Delta e^{u\mathbf{A}}\mathbf{B}\,du \approx \Delta \mathbf{B}.
\end{align}
For fixed coefficients and zero initial state, unrolling the discrete recurrence gives the following convolutional form:
\begin{align}
\mathbf{y} = \mathbf{x} \ast \mathbf{K}, \quad 
\mathbf{K} = [\mathbf{C}\bar{\mathbf{B}}, \mathbf{C}\bar{\mathbf{A}}\bar{\mathbf{B}}, \dots, \mathbf{C}\bar{\mathbf{A}}^{L-1}\bar{\mathbf{B}}],
\end{align}
providing a principled mechanism for modeling long-range temporal dependencies.

\textbf{Selective State-Space Models.} Recent advances such as Mamba~\citep{gu2023mamba} extend classical SSMs by making $\mathbf{B}$, $\mathbf{C}$, and the step size $\Delta$ input-dependent. This enables adaptive temporal dynamics, allowing the system to adjust its memory and evolution across different regimes.
Despite these advantages, standard SSMs treat inputs as unstructured sequences and do not explicitly incorporate spatial interactions defined by the underlying graph.

\textbf{Koopman Operators.} Koopman operator theory provides a linear perspective for analyzing nonlinear dynamical systems by lifting them into a space of observables. Let the system state be $\mathbf{h}(t) \in \mathbb{R}^{d_h}$ evolving under nonlinear dynamics $\dot{\mathbf{h}}(t) = f(\mathbf{h}(t))$. Instead of modeling the state directly, Koopman theory considers an observable mapping $\boldsymbol{\phi}(\mathbf{h})$ that lifts the state into a higher-dimensional space. The evolution of observables is governed by a linear operator $\mathcal{K}$ as follows:

\begin{align}
    \frac{d}{dt} \boldsymbol{\phi}(\mathbf{h}(t)) = \mathcal{K} \boldsymbol{\phi}(\mathbf{h}(t)).
    \label{eq:koopman_continuous}
\end{align}

In practice, the infinite-dimensional Koopman operator is approximated using a finite-dimensional representation. Let $\mathbf{z}(t) = \boldsymbol{\phi}(\mathbf{h}(t)) \in \mathbb{R}^{d_z}$ denote the lifted state. The dynamics can then be expressed as a linear system of dynamical systems as follows:
\begin{align}
    \dot{\mathbf{z}}(t) \approx \mathbf{K} \mathbf{z}(t),
    \label{eq:koopman_linear}
\end{align}
where $\mathbf{K} \in \mathbb{R}^{d_z \times d_z}$ is a finite-dimensional approximation of the Koopman generator.

\textbf{Discrete-Time Formulation.} Under a discrete-time setting with step size $\Delta$, the discrete dynamics can be written as follows:
\begin{align}
    \mathbf{z}[k] = \bar{\mathbf{K}} \mathbf{z}[k-1], 
    \qquad \bar{\mathbf{K}} = e^{\Delta \mathbf{K}},
    \label{eq:koopman_discrete}
\end{align}
which closely resembles the linear state evolution in state-space models.

\begin{remark}
Both Koopman operators and state-space models describe system evolution through linear dynamics in a latent space. Classical Koopman methods rely on predefined or learned observables $\boldsymbol{\phi}$ and assume time-invariant dynamics, whereas modern SSMs admit input-dependent operators that adapt to time-varying systems. We leverage this connection to learn a latent representation with structured, input-dependent dynamics and then couple it with normalized-adjacency propagation to incorporate spatial interactions directly into the latent evolution \citep{das2004laplacian}.
\end{remark}

\textbf{Problem Setup (Trajectory Prediction).}
Let $\mathcal{T}$ denote a collection of spatio-temporal trajectories sampled from a data distribution $p_{\mathrm{data}}$. Each trajectory is represented as a sequence of graph-structured states $\{\mathcal{G}_t\}_{t=0}^{T+\Delta T}$. We interpret the observed past trajectory segment $\{\mathcal{G}_t\}_{t=0}^{T}$ as a time-dependent function over the graph, and similarly the future segment $\{\mathcal{G}_t\}_{t=T+1}^{T+\Delta T}$ as its continuation. In addition to mapping trajectories, the model implicitly summarizes the past trajectory into a latent state representation that captures the system's history. Our goal is to learn a neural operator $\mathcal{F}_\theta : \{\mathcal{G}_t\}_{t=0}^{T} \;\mapsto\; \{\mathcal{G}_t\}_{t=T+1}^{T+\Delta T}$ that maps the past trajectory function to the future trajectory function, approximating the ground-truth solution operator $\mathcal{F}^\dagger$ and predicting the future evolution. The learning objective minimizes trajectory discrepancy:

\begin{align}
    \label{eq:objective}
    \min_{\theta} \; 
    \mathbb{E}_{\{\mathcal{G}_t\}_{t=0}^{T} \sim p_{\mathrm{data}}}
    \Big[
    \mathcal{L}\!\left(
    \mathcal{F}_\theta\!\left(\{\mathcal{G}_t\}_{t=0}^{T}\right),
    \mathcal{F}^\dagger\!\left(\{\mathcal{G}_t\}_{t=0}^{T}\right)
    \right)
    \Big],
\end{align}
where $\mathcal{L}$ measures trajectory discrepancy using the $L_2$ norm.

\begin{figure*}[t]
    \centering
    \includegraphics[width=0.9\linewidth]{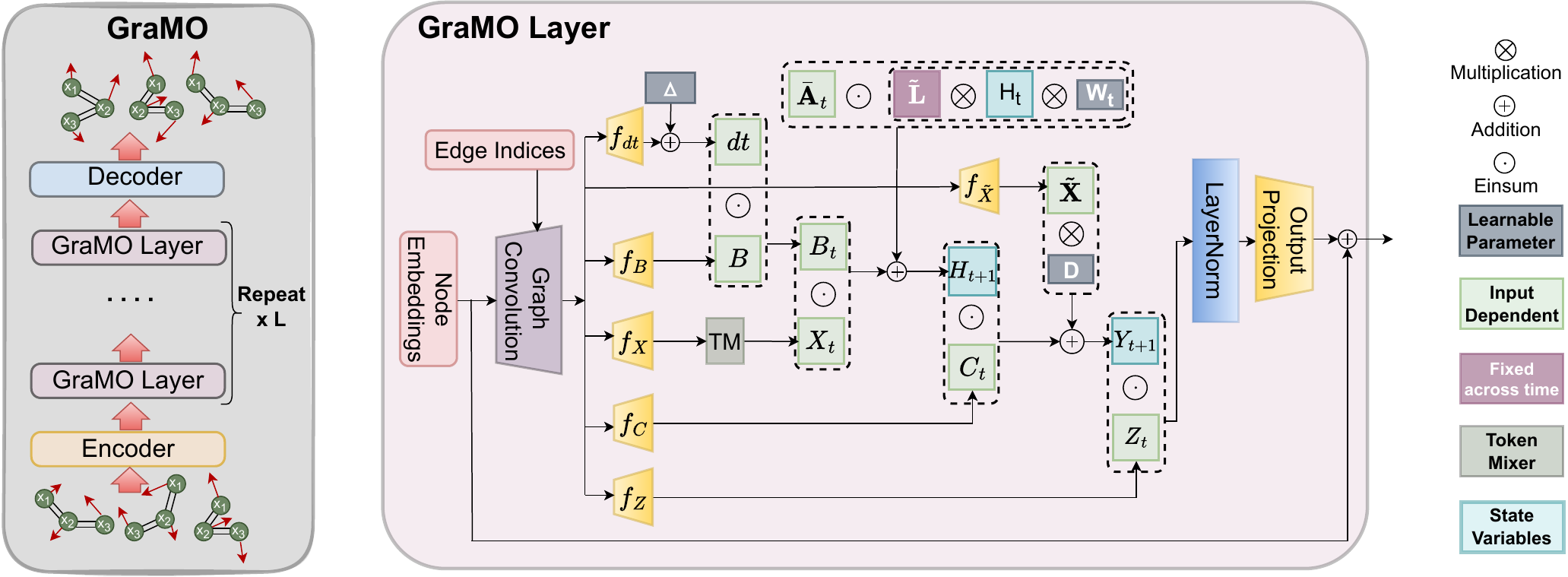}
    \caption{\label{fig:main_architecture}  \textbf{Overview.} The model maps past graph sequences $\{\mathcal{G}_t\}_{t=0}^{T}$ to future trajectories $\{\mathcal{G}_t\}_{t=T+1}^{T+\Delta T}$ using $L$ stacked \method blocks. (\textbf{Left}) Overall pipeline. (\textbf{Right}) A single \method layer showing the discrete update and gated skip connection (see Appendix Algorithm~\ref{alg:gramo_block}).}
\end{figure*}

\section{Methodology} \label{sec:method}
We introduce \method, a neural operator that maps past graph trajectories to future trajectories. \method captures temporal dynamics, graph-based interactions, memory of past states, and time-varying behavior through a unified graph-coupled state-space formulation.

\subsection{\method Formulation}
\textbf{Framework.} Consider a dynamical system on a graph $\mathcal{G} = (\mathcal{V}, \mathcal{E})$ with $V = |\mathcal{V}|$ nodes. Let $\mathbf{X}(t) \in \mathbb{R}^{V \times D}$ denote the input node features at continuous time $t$, where $D$ is the feature dimension. In our setup, $\mathbf{X}(t) = [\mathbf{x}(t) \| \mathbf{v}(t) \| \mathbf{z}(t)]$ concatenates positions $\mathbf{x}(t) \in \mathbb{R}^{V \times 3}$, velocities $\mathbf{v}(t) \in \mathbb{R}^{V \times 3}$, and node attributes $\mathbf{z}(t) \in \mathbb{R}^{V \times D_z}$, yielding $D = 6 + D_z$ for general system and let graph connectivity is specified by an adjacency matrix
$\mathbf{A}_{\mathrm{adj}} \in \mathbb{R}^{V \times V}$, constructed from domain-specific interactions. We introduce a latent state $\mathbf{H}(t) \in \mathbb{R}^{V \times D \times N}$ motivated by the ODE below:
\begin{equation}
    \label{eq:general_framework}
    \frac{d\mathbf{H}(t)}{dt} = \mathcal{F}\bigl(\mathbf{H}(t), \mathbf{X}(t), \mathcal{G}\bigr),
\end{equation}
where feature dimension $D$ preserves physical semantics, $N$ is the state dimension that provides a memory axis that enables long-range temporal integration, and $\mathcal{F}$ is an operator coupling spatial graph structure with temporal state dynamics. The output is obtained via a readout $\mathbf{Y}(t) = \mathcal{R}(\mathbf{H}(t))$.

The continuous-time formulation motivates three features: (i) \emph{continuous latent evolution}, where the ODE describes how $\mathbf{H}(t)$ evolves between observations; (ii) \emph{persistent memory}, where the discrete update carries information from earlier observations through its $N$ memory slots per node and feature channel; and (iii) \emph{adaptive dynamics}, where the current input modulates the latent-state transition.

\textbf{Motivation.} The design is motivated by three complementary ideas. From the SSM perspective, a learned latent state with structured HiPPO-initialized dynamics~\citep{gu2020hippo} motivates long-range memory. From the operator-theoretic perspective, lifting nonlinear dynamics to a higher-dimensional space where evolution is linear, as in Koopman analysis, offers a principled route to long-term prediction without recursive updates in the observation space. From the graph-learning perspective, message passing provides an inductive bias for modeling local interactions among particles through the underlying graph topology. \method draws on these ideas: the recurrence is affine in the latent state for fixed inputs, the graph operator propagates information across interacting nodes, and the per-step operator depends on the current input, generalizing classical operator approaches.

\textbf{Proposed Method.} We instantiate the general framework with a specific architecture that \emph{couples} spatial graph propagation with temporal state evolution. As a continuous-time motivation, consider:
\begin{equation}
    \label{eq:gramo_continuous}
    \frac{d\mathbf{H}(t)}{dt} = \mathbf{A} \odot \Bigl( \tilde{\mathbf{L}} \, \mathbf{H}(t) \, \mathbf{W}_s \Bigr) + \mathbf{B}(t) \odot \mathbf{X}(t),
\end{equation}
where $\tilde{\mathbf{L}} \in \mathbb{R}^{V \times V}$ is the normalized adjacency matrix with self-loops defined as $\tilde{\mathbf{L}} = \mathbf{D}_{\!\deg}^{-1/2}(\mathbf{A}_{\mathrm{adj}} + \mathbf{I})\mathbf{D}_{\!\deg}^{-1/2}$, acting on the node dimension. The matrix $\mathbf{W}_s \in \mathbb{R}^{N \times N}$ is a learnable spatial mixing matrix acting on the state dimension, $\mathbf{A} \in \mathbb{R}^{D \times N}$ with $\mathbf{A} = -\exp(\mathbf{A}_{\log})$ parameterizes negative elementwise coefficients and is broadcast across nodes dimension, and $\mathbf{B}(t) \in \mathbb{R}^{V \times D \times N}$ is an input-dependent projection where $\mathbf{X}(t)$ is broadcast along the state dimension. The output at each time instant is obtained by contracting over the state dimension as follows:
\begin{equation}
    \label{eq:gramo_output}
    \mathbf{Y}(t) = \mathbf{H}(t) \cdot \mathbf{C}(t),
\end{equation}
where $\mathbf{C}(t) \in \mathbb{R}^{V \times N}$ is the output projection and the contraction yields $\mathbf{Y}(t) \in \mathbb{R}^{V \times D}$. Crucially, $\mathbf{H}(t)$ is not reset at each timestep, allowing the model to integrate information over time. This does not eliminate prediction-error accumulation when predicted observations are fed back.

\textbf{Interpretation.} Equation~\eqref{eq:gramo_continuous} decomposes into two coupled components: (i) \emph{Temporally modulated graph propagation} ($\mathbf{A} \odot (\tilde{\mathbf{L}} \mathbf{H} \mathbf{W}_s)$), where the latent state first propagates across the graph via $\tilde{\mathbf{L}} \mathbf{H}$, then mixes along the state dimension via $\mathbf{W}_s$, and finally undergoes temporal modulation via $\mathbf{A}$; and (ii) \emph{Input injection} ($\mathbf{B} \odot \mathbf{X}$), which projects the current input features into the latent state space through the input-dependent $\mathbf{B}$. This motivates coupling state-space memory with graph propagation. Negative entries of $\mathbf{A}$ alone do not imply stability of this full continuous-time generator.

\subsection{Discretization and Implementation}
\textbf{Discrete graph-coupled update.} Motivated by Equation~\eqref{eq:gramo_continuous} and selective SSM parameterizations, we define the following discrete graph-coupled
update with positive, input-dependent scales
$\Delta_t \in \mathbb{R}_{>0}^{V \times D}$ as follows: \footnote{This update differs from exact ZOH discretization of Equation~\eqref{eq:gramo_continuous}, which exponentiates the full graph-coupled generator and integrates the input term over each interval.}

\begin{equation}
    \label{eq:gramo_discrete}
    \mathbf{H}_{t+1} = \bar{\mathbf{A}}_t \odot \Bigl( \tilde{\mathbf{L}} \, \mathbf{H}_t \, \mathbf{W}_s \Bigr) + \bar{\mathbf{B}}_t \odot \mathbf{X}_t,
\end{equation}
where $\bar{\mathbf{A}}_t = \exp(\Delta_t \odot \mathbf{A}) \in \mathbb{R}^{V \times D \times N}$ and $\bar{\mathbf{B}}_t = \Delta_t \odot \mathbf{B}_t \in \mathbb{R}^{V \times D \times N}$ are discrete gains inspired by the exponential transition and first-order input scaling in Equation~\eqref{eq:ssm_discrete}, with $\Delta_t$ broadcast along the state dimension. Following~\citet{gu2023mamba}, we initialize the underlying temporal parameters using the HiPPO framework~\citep{gu2020hippo}, which provides principled initialization for long-range dependency modeling. The update admits an elegant interpretation: at each step, the latent state is (i) propagated spatially via $\tilde{\mathbf{L}} \mathbf{H}_t$, (ii) mixed along the state dimension via $\mathbf{W}_s$, (iii) modulated temporally via $\bar{\mathbf{A}}_t$, and (iv) injected with input features via $\bar{\mathbf{B}}_t \odot \mathbf{X}_t$. 

\textbf{The \method Interpretation.} The recurrence in Equation~\eqref{eq:gramo_discrete} is affine in the latent state $\mathbf{H}_t$ for fixed inputs, with operator coefficients that depend on the current input $\mathbf{X}_t$. This structure is reminiscent of finite-dimensional Koopman lifts, in which a learned observable map renders nonlinear dynamics linear in a lifted space; \method differs from the classical formulation by allowing the lifted operator to vary with the input, which is needed for non-autonomous and regime-switching systems.

The latent state $\mathbf{H}_t \in \mathbb{R}^{V \times D \times N}$ stores, for each of the $V$ particles, a $D$-dimensional feature across $N$ memory slots. For fixed inputs, the per-step evolution is affine in this lifted representation, structurally similar to the role of Koopman observables but with an input-dependent operator. The factorization couples space ($V$), physical features ($D$), and temporal memory ($N$) within a single tensor. The update decomposes into several components, each with a specific modeling role, as described below.

\textbf{Spatial step --- $\tilde{\mathbf{L}}\,\mathbf{H}_t\,\mathbf{W}_s$.}The normalized adjacency $\tilde{\mathbf{L}}$ propagates information along edges of $\mathcal{G}$, while $\mathbf{W}_s$ mixes the $N$ memory slots. Together, they realize the joint operator $(\tilde{\mathbf{L}} \otimes \mathbf{I}_D \otimes \mathbf{W}_s^\top)$ that couples multi-hop spatial interactions with cross-scale memory mixing.

\textbf{Temporal step --- $\bar{\mathbf{A}}_t = \exp(\Delta_t \odot \mathbf{A})$.} HiPPO initialization of $\mathbf{A}$ provides structured long-range memory; the input-dependent step size $\Delta_t$ stretches it during slow phases and compresses it during fast transitions (e.g., contact), providing input-adaptive temporal gains. Uniform contraction additionally requires the bound as shown in Appendix Lemma~\ref{lemma:kt_bound}.

\textbf{Input drive --- $\bar{\mathbf{B}}_t \odot \mathbf{X}_t$.} The graph-aware, input-dependent projection $\mathbf{B}_t$ acts as a content-aware gate, selectively writing relevant features of the current observation into the latent state.

\textbf{Unified view.} For fixed inputs, the update is affine in the latent state, with a linear homogeneous transition. Stack entries in node--feature--memory order (memory index fastest), and set $\mathbf{h}_t=\operatorname{vec}(\mathbf{H}_t)$. Define the joint spatial--memory propagation operator $\mathbf{S}$,
the diagonal matrix of input-dependent temporal gates $\mathbf{D}_t$, and the resulting latent-state transition matrix $\mathbf{K}_t$ as follows:
\begin{equation}
 \mathbf{S}=\tilde{\mathbf{L}}\otimes\mathbf{I}_D\otimes\mathbf{W}_s^\top,
 \qquad \mathbf{D}_t=\operatorname{diag}(\operatorname{vec}\bar{\mathbf{A}}_t),
 \qquad \mathbf{K}_t=\mathbf{D}_t\mathbf{S}.
 \label{eq:correct_transition}
\end{equation}
Thus $\mathbf{h}_{t+1}=\mathbf{K}_t\mathbf{h}_t+\mathbf{u}_t$, where
$\mathbf{u}_t=\operatorname{vec}(\bar{\mathbf{B}}_t\odot\mathbf{X}_t)$, with the stated broadcasting. The matrix $\mathbf{K}_t\in\mathbb{R}^{VDN\times VDN}$ couples graph propagation and memory mixing with coordinatewise gating.

\textbf{Selectivity.} A key feature of \method is that the parameters $\mathbf{B}_t$, $\mathbf{C}_t$, and $\Delta_t$ are \emph{input-dependent}, generated dynamically from $\mathbf{X}_t$ via graph convolution as follows:
\begin{equation}
    \label{eq:selectivity}
    \mathbf{B}_t, \mathbf{C}_t, \boldsymbol{\delta}_t = f_{\theta}(\mathbf{X}_t, \mathcal{G}), \quad \Delta_t = \mathrm{softplus}(\boldsymbol{\delta}_t + \mathbf{b}_\Delta),
\end{equation}
where $f_{\theta}$ is a learnable graph neural network and $\mathbf{b}_\Delta \in \mathbb{R}^{D}$ is a learnable bias. This \emph{selective mechanism}~\citep{gu2023mamba} enables content-aware gating where $\mathbf{B}_t$ controls which input information enters the state, adaptive readout where $\mathbf{C}_t$ determines which state components contribute to output, and input-dependent adaptation where $\Delta_t$ modulates the temporal dynamics based on input.

\textbf{Gated Output with Skip Connection.} Following~\citet{gu2023mamba}, we apply a gated output mechanism with a residual skip connection. The graph module produces a gate $\mathbf{Z}_t$ and an auxiliary input $\tilde{\mathbf{X}}_t$, which modulate the state readout via element-wise gating and normalization as follows:
\begin{equation}
    \mathbf{O}_t
    =
    \mathrm{LayerNorm}\!\left(
    \bigl(\mathbf{Y}_t + \mathbf{D}_{\!\mathrm{skip}} \odot \tilde{\mathbf{X}}_t\bigr)
    \odot \sigma(\mathbf{Z}_t)
    \right)
    \mathbf{W}_{\mathrm{out}},
\end{equation}
where $\mathbf{D}_{\!\mathrm{skip}}$ is a learnable parameter, $\mathbf{W}_{\mathrm{out}}$ is the projection matrix and $\sigma$ denotes activation function.

\textbf{Full Architecture.} The model stacks $L$ coupled graph-SSM layers. Each layer $\ell$ applies the discrete update in Equation~\eqref{eq:gramo_discrete}, with residual connections to facilitate optimization. We use a bidirectional SSM formulation, where the graph-coupled update is applied in forward and backward temporal directions to capture richer trajectory context. Figure~\ref{fig:main_architecture} shows the overall architecture.

\subsection{Theoretical Analysis and Insights}
We summarize the key theoretical properties of \method below; full statements and proofs are deferred to Appendix Section ~\ref{app_sec:theory}. The results cover three aspects: an affine latent recurrence, conditional stability and graph-path propagation, and an adaptive graph moving-average representation.

\textbf{Affine latent recurrence.}
For a fixed input sequence, \method has an affine latent recurrence with an input-conditioned linear transition (see Appendix Theorem~\ref{thm:input_conditioned_koopman}). This provides a Koopman-inspired interpretation, without asserting an exact invariant Koopman observable space.

\textbf{Stability and long-range propagation.}
Appendix Proposition~\ref{prop:walk_accumulation} expresses propagation as sums over graph and memory paths, with gates retained in their correct order. Under the explicit uniform contraction condition of Lemma~\ref{lemma:kt_bound}, Proposition~\ref{prop:jacobian_multi_step} bounds the conditional state Jacobian. The normalized adjacency alone does not yield a convergent Neumann series; such a series applies to a fixed contractive full transition (Appendix Section ~\ref{app_sec:theory}).

\textbf{Adaptive graph moving-average representation.} This is an input-conditioned graph state-space recurrence with a moving-average expansion. Its coefficients may vary along an input trajectory (Proposition~\ref{prop:non_stationary_temporal_dynamics}); it is not generally a classical graph ARMA model with fixed polynomial graph filters. These results concern each directional recurrence with inputs held fixed, not the complete residual network or observation-feedback rollout.

\begin{theorem}[Adaptive Graph Moving-Average Representation]\label{thm:ma_main}
Let $\mathbf{K}_t$ and $\mathbf{u}_t$ be defined above. Suppose
$\sup_{t\in\mathbb{Z}}\|\mathbf{K}_t\|_2\le q<1$ and
$M=\sup_{t\in\mathbb{Z}}\|\mathbf{u}_t\|_2<\infty$.
The unique bounded entire solution has the absolutely convergent representation as follows:
\begin{equation}
 \mathbf{h}_t=\sum_{j=0}^{\infty}
 (\mathbf{K}_{t-1}\mathbf{K}_{t-2}\cdots\mathbf{K}_{t-j})\mathbf{u}_{t-1-j},
\end{equation}
where the product is the identity when $j=0$. For a finite initial time, the exact representation includes the initial-state term (see Theorem~\ref{theorem:graph_arma}).
\end{theorem}

\begin{table*}[t]
    \centering
    \caption{\textbf{Evaluation on N-body Simulation datasets.} AMSE ($\times 10^{-1}$) and FMSE ($\times 10^{-1}$) on N-body simulation. \one{First} and \two{Second} denote the best and second-best results.}
    \label{tab:nbody_main}
    \resizebox{\textwidth}{!}{%
    \begin{tabular}{lccc|ccc|ccc}
    \toprule
    \multirow{2}{*}{\textbf{Model}}
        & \multicolumn{3}{c|}{\textbf{Charged Particles}}
        & \multicolumn{3}{c|}{\textbf{Spring Dynamics}}
        & \multicolumn{3}{c}{\textbf{Gravity System}} \\
    \cmidrule(lr){2-4} \cmidrule(lr){5-7} \cmidrule(lr){8-10}
     & \textbf{AMSE} & \textbf{FMSE} & \textbf{Average}
     & \textbf{AMSE} & \textbf{FMSE} & \textbf{Average}
     & \textbf{AMSE} & \textbf{FMSE} & \textbf{Average} \\
    \midrule
    ST-GNN
    & $4.860_{\pm 0.240}$ & $10.340_{\pm 0.520}$ & $7.600_{\pm 0.380}$
    & $0.147_{\pm 0.007}$ & $0.382_{\pm 0.019}$ & $0.265_{\pm 0.013}$
    & $8.040_{\pm 0.400}$ & $16.150_{\pm 0.810}$ & $12.095_{\pm 0.605}$ \\
    Hamiltonian
    & $1.203_{\pm 0.260}$ & $2.832_{\pm 0.142}$ & $2.018_{\pm 0.201}$
    & \two{$0.015_{\pm 0.002}$} & \two{$0.032_{\pm 0.003}$} & \two{$0.024_{\pm 0.003}$}
    & \two{$2.530_{\pm 0.210}$} & $6.211_{\pm 0.210}$ & $4.371_{\pm 0.210}$ \\
    ST-TFN
    & $3.360_{\pm 0.170}$ & $7.420_{\pm 0.370}$ & $5.390_{\pm 0.270}$
    & $1.000_{\pm 0.050}$ & $2.410_{\pm 0.120}$ & $1.705_{\pm 0.085}$
    & $3.320_{\pm 0.170}$ & $7.490_{\pm 0.370}$ & $5.405_{\pm 0.270}$ \\
    ST-SE(3)TR
    & $4.020_{\pm 0.200}$ & $9.240_{\pm 0.460}$ & $6.630_{\pm 0.330}$
    & $0.880_{\pm 0.040}$ & $2.000_{\pm 0.100}$ & $1.440_{\pm 0.070}$
    & $3.440_{\pm 0.170}$ & $8.190_{\pm 0.410}$ & $5.815_{\pm 0.290}$ \\
    ST-EGNN
    & $1.900_{\pm 0.100}$ & $4.180_{\pm 0.210}$ & $3.040_{\pm 0.155}$
    & $0.100_{\pm 0.005}$ & $0.226_{\pm 0.011}$ & $0.163_{\pm 0.008}$
    & $3.150_{\pm 0.160}$ & $7.000_{\pm 0.350}$ & $5.075_{\pm 0.255}$ \\
    Koopman
    & $1.780_{\pm 0.090}$ & $3.340_{\pm 0.170}$ & $2.560_{\pm 0.130}$
    & $0.376_{\pm 0.019}$ & $1.100_{\pm 0.060}$ & $0.738_{\pm 0.040}$
    & $3.290_{\pm 0.160}$ & $8.220_{\pm 0.410}$ & $5.755_{\pm 0.285}$ \\
    EqMotion
    & $1.430_{\pm 0.070}$ & $3.040_{\pm 0.150}$ & $2.235_{\pm 0.110}$
    & $0.136_{\pm 0.007}$ & $0.365_{\pm 0.018}$ & $0.251_{\pm 0.013}$
    & $3.070_{\pm 0.150}$ & $6.600_{\pm 0.330}$ & $4.835_{\pm 0.240}$ \\
    SVAE
    & $3.720_{\pm 0.190}$ & $7.400_{\pm 0.370}$ & $5.560_{\pm 0.280}$
    & $0.122_{\pm 0.006}$ & $0.213_{\pm 0.011}$ & $0.168_{\pm 0.009}$
    & $5.900_{\pm 0.300}$ & $10.850_{\pm 0.540}$ & $8.375_{\pm 0.420}$ \\
    GeoTDM
    & $1.130_{\pm 0.060}$ & \two{$2.540_{\pm 0.130}$} & \two{$1.835_{\pm 0.095}$}
    & $0.031_{\pm 0.002}$ & $0.077_{\pm 0.004}$ & $0.054_{\pm 0.003}$
    & $2.610_{\pm 0.130}$ & \two{$6.020_{\pm 0.300}$} & \two{$4.315_{\pm 0.215}$} \\
    GraphMamba
    & $2.900_{\pm 0.110}$ & $5.920_{\pm 0.270}$ & $4.410_{\pm 0.190}$
    & $0.500_{\pm 0.006}$ & $0.527_{\pm 0.012}$ & $0.514_{\pm 0.009}$
    & $7.150_{\pm 0.260}$ & $9.000_{\pm 0.450}$ & $8.075_{\pm 0.355}$ \\
    HOPE
    & $1.540_{\pm 0.210}$ & $3.720_{\pm 0.430}$ & $2.630_{\pm 0.320}$
    & $0.029_{\pm 0.003}$ & $0.052_{\pm 0.011}$ & $0.041_{\pm 0.007}$
    & $3.250_{\pm 0.140}$ & $7.021_{\pm 0.150}$ & $5.136_{\pm 0.145}$ \\
    ESTAG
    & \two{$1.090_{\pm 0.140}$} & $2.830_{\pm 0.340}$ & $1.960_{\pm 0.240}$
    & $0.045_{\pm 0.004}$ & $0.063_{\pm 0.012}$ & $0.054_{\pm 0.008}$
    & $3.150_{\pm 0.170}$ & $8.210_{\pm 0.340}$ & $5.680_{\pm 0.255}$ \\
    EGNO
    & $2.090_{\pm 0.150}$ & $4.830_{\pm 0.340}$ & $3.460_{\pm 0.245}$
    & $0.045_{\pm 0.002}$ & $0.058_{\pm 0.012}$ & $0.052_{\pm 0.007}$
    & $4.160_{\pm 0.110}$ & $9.110_{\pm 0.140}$ & $6.635_{\pm 0.125}$ \\
    \midrule
    \textbf{\method}
    & \one{$0.790_{\pm 0.040}$} & \one{$2.410_{\pm 0.120}$} & \one{$1.600_{\pm 0.080}$}
    & \one{$0.011_{\pm 0.002}$} & \one{$0.024_{\pm 0.004}$} & \one{$0.018_{\pm 0.002}$}
    & \one{$1.290_{\pm 0.060}$} & \one{$3.720_{\pm 0.190}$} & \one{$2.505_{\pm 0.125}$} \\
    \bottomrule
    \end{tabular}
    }
\end{table*}

\section{Numerical Experiments}
In this section, we evaluate \method on a diverse set of physical dynamical systems, including N-body systems (charged, gravitational, and spring interactions), motion capture (MoCap), and robotics datasets (rope control and soft robots), to assess the effectiveness of our core design choices.

\textbf{Benchmark Details.} We evaluate \method on diverse trajectory prediction tasks across articulated dynamical systems, including CMU Motion Capture (Walk \#35, Run \#9)~\citep{cmu2003motion,huang2022equivariant,han2022equivariant}, N-body simulations (charged particles, spring systems, and gravitational systems)~\citep{brandstetter2021geometric,kipf2018neural,satorras2021n}, and robotics datasets (rope control and soft robots)~\citep{li2019learning}.

\textbf{Baselines.} We compare \method against a broad set of representative baselines spanning spatio-temporal graph networks, neural operator based including STGCN \citep{yu2017spatio}, TFN \citep{thomas2018tensor}, RF \citep{kohler2019equivariant}, Compositional Koopman Operator \citep{li2019learning}, Hamiltonian \citep{samuel2019hamiltonian}, SE(3)-Transformer \citep{fuchs2020se}, STEGNN \citep{satorras2021en}, AGL-STAN \citep{sun2022spatial}, SVAE \citep{xu2022socialvae}, EqMotion \citep{xu2023eqmotion}, HOPE \citep{luo2023hope}, ESTAG \citep{wu2023equivariant}, GraphODE \citep{luo2023graph}, GraphMamba \citep{wang2024graph}, GeoTDM \citep{han2024geometric}, EGNO\citep{xu2024equivariant} and NS-GNN \citep{yuan2025non}. 

\textbf{Evaluation Metrics.} We evaluate performance under two settings: trajectory-to-state (\ttos) and trajectory-to-trajectory (\ttot). In the trajectory-to-state setting, performance is measured at the final time step using the Final Mean Squared Error (FMSE), defined as $\text{FMSE} = \lVert \mathbf{X}(t_T) - \mathbf{X}^{\dagger}(t_T) \rVert^2,$
where $\mathbf{X}^{\dagger}$ denotes the ground-truth state.  In contrast, the trajectory-to-trajectory setting evaluates the full rollout using the Average Mean Squared Error (AMSE), $\text{AMSE} = \frac{1}{T}\sum_{t=1}^{T} \lVert \mathbf{X}(t) - \mathbf{X}^{\dagger}(t) \rVert^2,$ which is used as the evaluation metric unless stated otherwise.

\textbf{Training Details.} All models are trained using the Adam optimizer~\citep{kingma2014adam} with the StepLR scheduler. Experiments are conducted on Ubuntu 20.04.3 LTS, equipped with an Intel(R) Core(TM) i9-10900X CPU and a single NVIDIA RTX A6000 GPU with 48 GB of memory.

\subsection{N-Body Simulation}
\textbf{Experimental Setup.} We consider three scenarios in the N-body simulation datasets. (i) \emph{Charged Particles} ~\citep{kipf2018neural, satorras2021n}, where $N = 5$ particles with charges randomly chosen between $+1/-1$ move under Coulomb forces. (ii) \emph{Spring Dynamics}~\citep{kipf2018neural}, where $N = 5$ particles with random masses are connected by springs with a probability of $0.5$ between each pair, and the forces follow Hooke’s law. 
(iii) \emph{Gravity System} ~\citep{brandstetter2021geometric}, where $N = 10$ particles with random masses and initial velocities move under gravitational forces. For all three datasets, we use $3000$ trajectories for training, $600$ for validation, and $600$ for testing. For each trajectory, we use $10$ frames as input and predict the next $20$ frames.

\textbf{Results.} Table~\ref{tab:nbody_main} reports the AMSE and FMSE of all methods across the three N-body simulation regimes, with values in $\times 10^{-1}$ units. Trajectory models, including the deterministic EqMotion and the probabilistic GeoTDM, generally outperform frame-to-frame predictors by mitigating the error accumulation inherent to iterative roll-out. Across all settings and metrics, \method attains the lowest error. On FMSE, it improves over the second-best baseline from $2.540 \rightarrow 2.410$ on \emph{Charged Particles} (GeoTDM), $0.032 \rightarrow 0.024$ on \emph{Spring Dynamics} (ESTAG), and $6.020 \rightarrow 3.720$ on \emph{Gravity System} (GeoTDM). On AMSE, the comparable gains are $1.090 \rightarrow 0.790$ on \emph{Charged}, $0.015 \rightarrow 0.011$ on \emph{Spring}, and $2.530 \rightarrow 1.290$ on \emph{Gravity}. The consistency of improvements across qualitatively distinct interaction types, including short-range electrostatic, harmonic spring, and long-range gravitational interactions, suggests that \method captures interaction dynamics in a way that generalizes beyond any single physical dynamical regime.

\begin{table*}[t]
    \centering
    \caption{\textbf{Evaluation on Robotics datasets.} AMSE ($\times 10^{-1}$) and FMSE ($\times 10^{-1}$) on Rope, Soft Robot, and Swim prediction. \one{First} and \two{Second} denote the best and second-best results.}
    \label{tab:robotics_fmse}
    \resizebox{\textwidth}{!}{%
    \begin{tabular}{lccc|ccc|ccc}
    \toprule
    \multirow{2}{*}{\textbf{Model}}
        & \multicolumn{3}{c|}{\textbf{Rope}}
        & \multicolumn{3}{c|}{\textbf{Soft}}
        & \multicolumn{3}{c}{\textbf{Swim}} \\
    \cmidrule(lr){2-4} \cmidrule(lr){5-7} \cmidrule(lr){8-10}
     & \textbf{AMSE} & \textbf{FMSE} & \textbf{Average}
     & \textbf{AMSE} & \textbf{FMSE} & \textbf{Average}
     & \textbf{AMSE} & \textbf{FMSE} & \textbf{Average} \\
    \midrule
    ST-GNN
    & $0.6710_{\pm 0.0550}$ & $1.4270_{\pm 0.1150}$ & $1.0490_{\pm 0.0637}$
    & $0.1320_{\pm 0.0110}$ & $0.1580_{\pm 0.0090}$ & $0.1450_{\pm 0.0071}$
    & $0.0400_{\pm 0.0040}$ & $0.0420_{\pm 0.0030}$ & $0.0410_{\pm 0.0025}$ \\
    ST-TFN
    & $0.8530_{\pm 0.0720}$ & $1.8090_{\pm 0.1670}$ & $1.3310_{\pm 0.0909}$
    & $0.1590_{\pm 0.0100}$ & $0.2090_{\pm 0.0110}$ & $0.1840_{\pm 0.0074}$
    & $0.0500_{\pm 0.0030}$ & $0.0550_{\pm 0.0030}$ & $0.0525_{\pm 0.0021}$ \\
    ST-SE(3)TR
    & $0.8570_{\pm 0.0800}$ & $1.8410_{\pm 0.1210}$ & $1.3490_{\pm 0.0725}$
    & $0.1640_{\pm 0.0110}$ & $0.2100_{\pm 0.0200}$ & $0.1870_{\pm 0.0114}$
    & $0.0500_{\pm 0.0030}$ & $0.0570_{\pm 0.0040}$ & $0.0535_{\pm 0.0025}$ \\
    ST-EGNN
    & $0.3110_{\pm 0.0200}$ & $0.6680_{\pm 0.0530}$ & $0.4895_{\pm 0.0283}$
    & $0.0620_{\pm 0.0040}$ & $0.0760_{\pm 0.0050}$ & $0.0690_{\pm 0.0032}$
    & $0.0180_{\pm 0.0010}$ & $0.0200_{\pm 0.0010}$ & $0.0190_{\pm 0.0007}$ \\
    Koopman
    & $1.2860_{\pm 0.0940}$ & $2.2330_{\pm 0.2630}$ & $1.7595_{\pm 0.1396}$
    & $0.3510_{\pm 0.0250}$ & $0.7100_{\pm 0.0380}$ & $0.5305_{\pm 0.0227}$
    & $0.1180_{\pm 0.0120}$ & $0.2500_{\pm 0.0190}$ & $0.1840_{\pm 0.0112}$ \\
    EqMotion
    & $0.3220_{\pm 0.0300}$ & $0.6770_{\pm 0.0340}$ & $0.4995_{\pm 0.0227}$
    & $0.0620_{\pm 0.0050}$ & $0.0790_{\pm 0.0060}$ & $0.0705_{\pm 0.0039}$
    & $0.0190_{\pm 0.0020}$ & $0.0210_{\pm 0.0010}$ & $0.0200_{\pm 0.0011}$ \\
    SVAE
    & $0.5370_{\pm 0.0390}$ & $0.9120_{\pm 0.0890}$ & $0.7245_{\pm 0.0486}$
    & $0.0990_{\pm 0.0060}$ & $0.1020_{\pm 0.0080}$ & $0.1005_{\pm 0.0050}$
    & $0.0310_{\pm 0.0030}$ & $0.0280_{\pm 0.0030}$ & $0.0295_{\pm 0.0021}$ \\
    GeoTDM
    & $0.2510_{\pm 0.0210}$ & \two{$0.5380_{\pm 0.0430}$} & $0.3945_{\pm 0.0239}$
    & \two{$0.0470_{\pm 0.0040}$} & \two{$0.0600_{\pm 0.0050}$} & \two{$0.0535_{\pm 0.0032}$}
    & \two{$0.0140_{\pm 0.0010}$} & \two{$0.0180_{\pm 0.0010}$} & \two{$0.0160_{\pm 0.0007}$} \\
    GraphMamba
    & $0.4120_{\pm 0.0300}$ & $0.9670_{\pm 0.0630}$ & $0.6895_{\pm 0.0349}$
    & $0.0920_{\pm 0.0050}$ & $0.1860_{\pm 0.0040}$ & $0.1390_{\pm 0.0032}$
    & $0.0280_{\pm 0.0020}$ & $0.0400_{\pm 0.0020}$ & $0.0340_{\pm 0.0014}$ \\
    ESTAG
    & \two{$0.2120_{\pm 0.0400}$} & $0.5630_{\pm 0.0430}$ & \two{$0.3875_{\pm 0.0294}$}
    & $0.0530_{\pm 0.0060}$ & $0.0820_{\pm 0.0030}$ & $0.0675_{\pm 0.0034}$
    & $0.0180_{\pm 0.0030}$ & $0.0540_{\pm 0.0020}$ & $0.0360_{\pm 0.0018}$ \\
    \midrule
    \textbf{\method}
    & \one{$0.1860_{\pm 0.0100}$} & \one{$0.4120_{\pm 0.0400}$} & \one{$0.2990_{\pm 0.0206}$}
    & \one{$0.0360_{\pm 0.0030}$} & \one{$0.0470_{\pm 0.0030}$} & \one{$0.0415_{\pm 0.0021}$}
    & \one{$0.0110_{\pm 0.0010}$} & \one{$0.0130_{\pm 0.0010}$} & \one{$0.0120_{\pm 0.0007}$} \\
    \bottomrule
    \end{tabular}
    }
\end{table*}

\subsection{Robotics Dataset}
\textbf{Experimental Setup.} We evaluate \method on robotics environments involving rope manipulation and soft-body control \citep{li2019learning}. We consider three settings: 
(i) \emph{Rope}, where the top mass of a rope is actuated horizontally while the remaining masses evolve under internal forces and gravity; 
(ii) \emph{Soft}, where a soft robot composed of deformable blocks with embedded actuators is controlled, with one block anchored to the ground; and 
(iii) \emph{Swim}, where the same soft robot operates in a fluid environment without anchoring. In the Rope setting, each mass is treated as a node with 2D position and velocity (4D per node). In the Soft and Swim settings, each quadrilateral element is treated as a node with positions and velocities of its four corners (16D per element). We use a 70:15:15 train/validation/test split, conditioning on 10 frames and predicting the next 20.

\textbf{Results.}
Table~\ref{tab:robotics_fmse} reports the AMSE and FMSE of all methods across the three robotics environments. Across all settings and metrics, \method attains the lowest error. On FMSE, it improves over the second-best baseline from $0.5380 \rightarrow 0.4120$ on \emph{Rope}, $0.0600 \rightarrow 0.0470$ on \emph{Soft}, and $0.0180 \rightarrow 0.0130$ on \emph{Swim} (all GeoTDM); the AMSE column shows comparable gains, with ESTAG narrowly displacing GeoTDM as the second-best on \emph{Rope}. Trajectory-level models such as GeoTDM, ESTAG, and EqMotion outperform frame-to-frame predictors (ST-GNN, ST-TFN, ST-SE(3)TR), reflecting the benefit of mitigating roll-out error accumulation in long-horizon control. The consistency of improvements across qualitatively distinct control regimes, including actuated rope manipulation, anchored soft-body deformation, and free-swimming locomotion in fluid, indicates that \method captures deformable dynamics in a way that generalizes beyond any single robotic setting.

\begin{table*}[t]
    \centering
    \caption{\textbf{Evaluation on the MoCap dataset.} FMSE for Walk ($\times 10^{-1}$) and Run ($\times 10^{0}$) cases under different time intervals. \one{First} and \two{Second} denote the best and second-best results. ``-'' indicates the model failed to converge during training.}
    \label{tab:mocap_fmse}
    \resizebox{\textwidth}{!}{%
    \begin{tabular}{@{}l*{5}{c}@{\qquad}*{5}{c}@{}}
        \toprule
        \multirow{2}{*}{\textbf{Model}}
        & \multicolumn{5}{c}{\textbf{MoCap Walk} ($\times 10^{-1}$)}
        & \multicolumn{5}{c}{\textbf{MoCap Run} ($\times 10^{0}$)} \\
        \cmidrule(lr){2-6} \cmidrule(lr){7-11}
        & \textbf{5} & \textbf{10} & \textbf{15} & \textbf{20} & \textbf{Average}
        & \textbf{5} & \textbf{10} & \textbf{15} & \textbf{20} & \textbf{Average} \\
        \midrule
        ST-GNN
        & $1.121_{\pm 0.159}$ & $1.224_{\pm 0.100}$ & $2.615_{\pm 0.478}$ & $3.359_{\pm 0.601}$ & $2.080_{\pm 0.335}$
        & $0.560_{\pm 0.107}$ & $1.160_{\pm 0.166}$ & $1.538_{\pm 0.234}$ & $1.779_{\pm 0.278}$ & $1.259_{\pm 0.196}$ \\
        ST-TFN
        & $0.238_{\pm 0.032}$ & $0.721_{\pm 0.038}$ & $1.320_{\pm 0.067}$ & $2.092_{\pm 0.094}$ & $1.093_{\pm 0.058}$
        & $0.396_{\pm 0.073}$ & $0.796_{\pm 0.054}$ & $1.708_{\pm 0.318}$ & $2.086_{\pm 0.133}$ & $1.246_{\pm 0.144}$ \\
        ST-SE(3)TR
        & $0.146_{\pm 0.017}$ & $0.376_{\pm 0.097}$ & $0.760_{\pm 0.161}$ & $1.119_{\pm 0.347}$ & $0.600_{\pm 0.155}$
        & $0.280_{\pm 0.045}$ & $0.700_{\pm 0.131}$ & $1.165_{\pm 0.267}$ & $1.732_{\pm 0.550}$ & $0.969_{\pm 0.248}$ \\
        ST-EGNN
        & $0.188_{\pm 0.026}$ & $0.591_{\pm 0.103}$ & $1.140_{\pm 0.123}$ & $2.097_{\pm 0.205}$ & $1.004_{\pm 0.114}$
        & $0.444_{\pm 0.072}$ & $1.082_{\pm 0.113}$ & $2.375_{\pm 0.202}$ & $3.784_{\pm 0.429}$ & $1.921_{\pm 0.204}$ \\
        EqMotion
        & -- & -- & -- & -- & --
        & $12.074_{\pm 3.074}$ & $15.299_{\pm 4.127}$ & $21.074_{\pm 2.073}$ & $18.604_{\pm 3.503}$ & $16.763_{\pm 3.194}$ \\
        STGCN
        & $0.302_{\pm 0.115}$ & $0.828_{\pm 0.203}$ & $1.516_{\pm 0.384}$ & $1.988_{\pm 0.206}$ & $1.159_{\pm 0.227}$
        & $0.131_{\pm 0.024}$ & $0.582_{\pm 0.121}$ & $1.101_{\pm 0.089}$ & $1.508_{\pm 0.176}$ & $0.831_{\pm 0.103}$ \\
        AGL-STAN
        & $1.729_{\pm 0.516}$ & $1.789_{\pm 0.673}$ & $2.030_{\pm 0.704}$ & $2.155_{\pm 0.763}$ & $1.926_{\pm 0.664}$
        & $0.511_{\pm 0.137}$ & $0.628_{\pm 0.191}$ & $0.648_{\pm 0.271}$ & \two{$0.831_{\pm 0.283}$} & $0.654_{\pm 0.221}$ \\
        HOPE
        & \two{$0.043_{\pm 0.003}$} & \two{$0.112_{\pm 0.010}$} & \two{$0.320_{\pm 0.017}$} & $0.887_{\pm 0.090}$ & \two{$0.341_{\pm 0.030}$}
        & $0.049_{\pm 0.002}$ & $0.270_{\pm 0.029}$ & $0.612_{\pm 0.037}$ & $0.972_{\pm 0.112}$ & \two{$0.476_{\pm 0.045}$} \\
        GraphODE
        & $0.047_{\pm 0.002}$ & $0.125_{\pm 0.003}$ & $0.386_{\pm 0.017}$ & $1.142_{\pm 0.063}$ & $0.425_{\pm 0.021}$
        & $0.058_{\pm 0.003}$ & $0.393_{\pm 0.001}$ & $0.897_{\pm 0.026}$ & $1.127_{\pm 0.023}$ & $0.619_{\pm 0.013}$ \\
        ESTAG
        & $0.054_{\pm 0.004}$ & $0.213_{\pm 0.012}$ & $0.530_{\pm 0.038}$ & $1.085_{\pm 0.070}$ & $0.471_{\pm 0.031}$
        & $0.041_{\pm 0.002}$ & $0.250_{\pm 0.019}$ & $0.771_{\pm 0.050}$ & $1.767_{\pm 0.251}$ & $0.707_{\pm 0.081}$ \\
        EGNO
        & $0.053_{\pm 0.001}$ & $0.156_{\pm 0.002}$ & $0.559_{\pm 0.024}$ & $2.963_{\pm 0.085}$ & $0.933_{\pm 0.028}$
        & $0.047_{\pm 0.002}$ & $0.287_{\pm 0.003}$ & $0.723_{\pm 0.076}$ & $1.316_{\pm 0.023}$ & $0.593_{\pm 0.026}$ \\
        NS-EGNN
        & $0.051_{\pm 0.002}$ & $0.166_{\pm 0.006}$ & $0.397_{\pm 0.037}$ & \two{$0.775_{\pm 0.085}$} & $0.347_{\pm 0.033}$
        & \two{$0.033_{\pm 0.002}$} & \two{$0.187_{\pm 0.009}$} & \two{$0.584_{\pm 0.076}$} & $1.226_{\pm 0.215}$ & $0.508_{\pm 0.076}$ \\
        \midrule
        \textbf{\method}
        & \one{$0.039_{\pm 0.003}$} & \one{$0.075_{\pm 0.006}$} & \one{$0.143_{\pm 0.012}$} & \one{$0.213_{\pm 0.005}$} & \one{$0.117_{\pm 0.007}$}
        & \one{$0.023_{\pm 0.002}$} & \one{$0.067_{\pm 0.009}$} & \one{$0.164_{\pm 0.016}$} & \one{$0.387_{\pm 0.022}$} & \one{$0.161_{\pm 0.012}$} \\
        \bottomrule
    \end{tabular}%
    }
\end{table*}

\subsection{Motion Capture Dataset}
\textbf{Experimental Setup.} We evaluate human motion dynamics using the \textit{walk} and \textit{run} subsets of the CMU Motion Capture Database \citep{cmu2003motion}. Trajectories are filtered for duration, and the task is formulated to predict a target frame based on an observation window of $10$ historical frames ($\Delta t=1$). We assess forecasting capability across varying horizons by introducing a \textit{time gap} between the input and target, specifically at timestamps of $\{5, 10, 15, 20\}$. Consistent with prior literature \citep{yuan2025non}, we apply data normalization to the \textit{run} sequences.

\begin{wraptable}{r}{0.55\textwidth}
    \centering
    \caption{\textbf{Ablation Study.} N-Body Charged (AMSE \(\times 10^{-1}\)) and MoCap-Run (FMSE \(\times 10^{-1}\)) datasets.}
    \label{main:ablation}
    \small
    \setlength{\tabcolsep}{6pt}
    \begin{tabular}{lcc}
    \toprule
    \textbf{Component} & \textbf{Charged} & \textbf{MoCap Run}\\
    \midrule
    \method w/o Temporal & $1.74_{\pm 0.05}$ & $5.23_{\pm 0.3}$ \\
    \method w/o BiSSM & $1.81_{\pm 0.03}$ & $4.57_{\pm 0.4}$ \\
    \method w/o Selectivity & $1.23_{\pm 0.01}$ & $5.92_{\pm 0.4}$ \\
    \method w/o HiPPO & \two{$0.93_{\pm 0.02}$} & \two{$4.23_{\pm 0.5}$} \\
    \midrule 
    \textbf{\method} & \one{$0.79_{\pm 0.04}$} & \one{$3.87_{\pm 0.2}$} \\
    \bottomrule
    \end{tabular}
\end{wraptable}

\textbf{Results.} Table~\ref{tab:mocap_fmse} reports the FMSE of all methods on the MoCap \emph{Walk} and \emph{Run} sequences. \method achieves the lowest error across all settings, reducing the average FMSE from $0.341 \rightarrow 0.117$ on \emph{Walk} and from $0.476 \rightarrow 0.161$ on \emph{Run} compared to the second-best baseline NS-EGNN. This corresponds to an average relative error reduction of $65.95\%$. The improvement is especially pronounced at longer prediction horizons (15ts, 20ts), indicating that the graph-coupled latent dynamics improve stability and accuracy for long-term motion forecasting.

\subsection{Ablations}
We conduct ablation studies to assess the contribution of each component in \method. Table~\ref{main:ablation} reports results on N-Body Charged and MoCap-Run. Removing temporal modeling leads to the largest degradation, increasing the error from $0.79$ to $1.74$ on Charged and from $3.87$ to $5.23$ on MoCap-Run, showing that temporal state evolution is essential for long-horizon prediction. Removing the bidirectional SSM also hurts performance, particularly on MoCap-Run, where the error increases from $3.87$ to $4.57$, highlighting the benefit of bidirectional temporal context for motion forecasting.

The selectivity mechanism is also important: without input-dependent parameters, the error rises to $1.23$ on Charged and $5.92$ on MoCap-Run, indicating that adaptive dynamics are useful for modeling changing system states. Finally, replacing the HiPPO initialization degrades performance to $0.93$ and $4.23$, confirming the benefit of structured long-range memory. Overall, the full \method achieves the best performance across both datasets, demonstrating that temporal modeling, selectivity, HiPPO memory, and bidirectional processing all contribute to the final model performance. We further evaluate zero-shot generalization to unseen system sizes. To assess whether the learned inductive bias transfers across particle counts, we train on the robotics rope dataset with $N \in \{5,7,8\}$, where $N$ denotes the number of particles/masses in the rope, and test on the held-out setting $N=9$. \method achieves an AMSE of $0.217$, compared to $0.525$ for the Compositional Koopman operator, indicating that the graph-conditioned SSM better extrapolates to unseen system sizes.

\textbf{Visual Demonstrations.} Figure~\ref{fig:main_efficiency} illustrates qualitative trajectory predictions on the MoCap dataset. The predicted sequences align well with the ground-truth motion, especially in terms of overall body configuration and torso movement. These visualizations suggest that \method effectively models coordinated joint behavior and maintains physically reasonable motion.

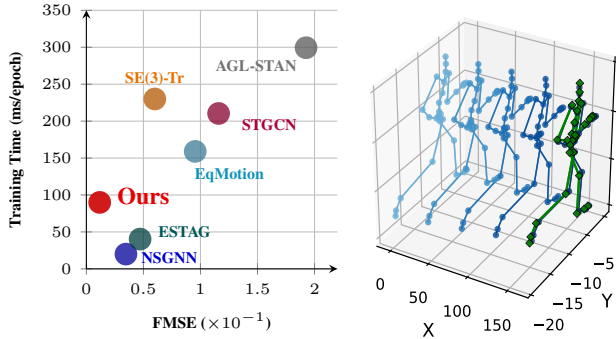
\begin{wrapfigure}{l}{0.60\textwidth}
    \centering
    \setlength{\tabcolsep}{0.2pt}
    \begin{tabular}{cc}
    \begin{tikzpicture}
    \begin{axis}[
        width=4.9cm,
        height=5cm,
        xlabel={\textbf{FMSE ($\times 10^{-1}$)}},
        ylabel={\textbf{Training Time (ms/epoch)}},
        xmin=0, xmax=2.2,
        ymin=0, ymax=350,
        grid=both,
        minor grid style={gray!25},
        major grid style={gray!50},
        xtick distance=0.5,
        ytick distance=50,
        ticklabel style={font=\tiny},
        label style={font=\tiny},
        axis lines=left,
        legend style={draw=none, font=\tiny},
        clip=false
    ]    
    \fill[opacity=0.75, blue!60!black] (axis cs:0.347, 20) circle (0.15cm);
    \node[font=\tiny, anchor=north west, blue!80!black] at (axis cs:0.4, 35) {\textbf{NSGNN}};
    \fill[opacity=0.75, teal!50!black] (axis cs:0.471, 40.3) circle (0.15cm);
    \node[font=\tiny, anchor=west, teal!70!black] at (axis cs:0.55, 50.3) {\textbf{ESTAG}};
    \fill[opacity=0.75, orange!70!black] (axis cs:0.600, 230) circle (0.15cm);
    \node[font=\tiny, anchor=south, orange!90!black] at (axis cs:0.600, 235) {\textbf{SE(3)-Tr}};
    \fill[opacity=0.75, cyan!50!black] (axis cs:0.954, 159) circle (0.15cm);
    \node[font=\tiny, anchor=south, cyan!70!black] at (axis cs:1.254, 104) {\textbf{EqMotion}};
    \fill[opacity=0.75, gray!70!black] (axis cs:1.926, 299.4) circle (0.15cm);
    \node[font=\tiny, anchor=north east, gray!90!black] at (axis cs:1.926, 295) {\textbf{AGL-STAN}};
    \fill[opacity=0.75, purple!70!black] (axis cs:1.159, 210.4) circle (0.15cm);
    \node[font=\tiny, anchor=north east, purple!90!black] at (axis cs:1.926, 215) {\textbf{STGCN}};
    \fill[opacity=0.9, red!80!black] (axis cs:0.117, 90) circle (0.15cm);
    \node[font=\small, anchor=south, red!90!black] at (axis cs:0.5, 75) {\textbf{Ours}};
    \end{axis}
    \end{tikzpicture}
    &
    \includegraphics[width=4.2cm,height=4.5cm]{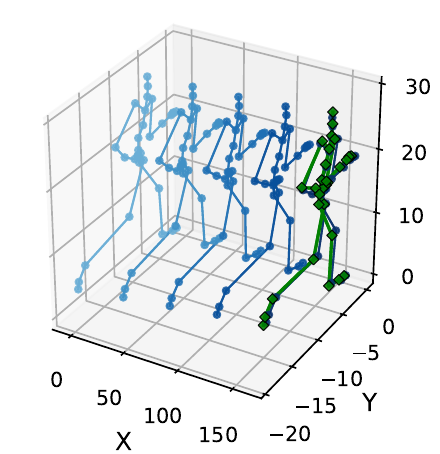}
    \end{tabular}
    \caption{\textbf{Efficiency and Visualization.} (\textbf{Left}) Training time versus FMSE ($\times 10^{-1}$) for baselines on MoCap (Walk) dataset. (\textbf{Right}) MoCap (Run) trajectory visualization, where predicted trajectories are shown with a \two{Blue} color gradient over time, and the ground-truth final state is shown in \one{Green}.}
    \label{fig:main_efficiency}
\end{wrapfigure}

\textbf{Efficiency.} Figure~\ref{fig:main_efficiency} compares the computational efficiency of \method on the MoCap dataset against several baseline methods. \method demonstrates favorable training efficiency relative to EqMotion, SE(3)-Transformer, STGCN, and ACL-STAN, while achieving lower FMSE on MoCap. Simpler equivariant models, such as NSGNN and ESTAG, train faster but exhibit lower predictive performance, suggesting a trade-off between computational efficiency and the ability to model long-range interaction dynamics. Furthermore, Table ~\ref{tab:nbody_main} - \ref{tab:robotics_fmse} shows that the proposed method outperforms artificial node-ordering approaches such as GraphMamba \citep{wang2024graph}, highlighting the importance of normalized-adjacency graph propagation for joint spatial-temporal modeling. It suggests that preserving the interaction topology is more effective than serializing nodes into an arbitrary sequence when modeling interacting particle dynamics \citep{mirzaev2013laplacian}.

\section{Conclusion and Limitations}
We presented \method, a neural operator that integrates state-space modeling with graph-based representations to capture long-range interactions in articulated dynamical systems. By jointly modeling temporal evolution and graph-structured dependencies, \method provides a principled framework for learning complex spatiotemporal dynamics. Experiments on N-body systems (charged, gravitational, and spring), motion capture, and robotics datasets demonstrate consistent improvements over baselines. 

\textbf{Limitations.} A limitation of the \method is its scaling with graph size due to repeated graph–state interactions. Exploring graph sparsification and subgraph sampling presents a promising direction to improve scalability while preserving performance.

\section*{Acknowledgement}
A start-up grant from the IISc supported this work in part by funding GPU computing infrastructure. Prathosh A P acknowledges support from the Infosys Foundation Young Investigator Award. N M Anoop Krishnan acknowledges support from the Alexander von Humboldt Foundation, the Google Research Scholar Award, Tower Research, and the Anusandhan National Research Foundation (grants ANRF/ARG/2025/007405/ENS and ANRF/AISE/2025/000706/AISSAI). Karn Tiwari acknowledges support from the Government of India through the Prime Minister’s Research Fellowship (PMRF) and travel grants from the Kotak IISc AI–ML Centre (KIAC) and Google.

\bibliographystyle{plainnat}
\small{\bibliography{references}}


\etocdepthtag.toc{mtmain}

\appendix
\etocdepthtag.toc{mtappendix}

\etocsettagdepth{mtmain}{none}
\etocsettagdepth{mtappendix}{subsection}

\etocsettocstyle{\section*{Appendix Contents}}{}

\newpage
\tableofcontents
\newpage

\section{Notation and Conventions}
\begin{table*}[!htb]
    \centering
    \caption{\label{tab:notations} \textbf{Summary of Notations.} 
    Notation grouped by general setup, state-space models, and \method-specific components.}
    \renewcommand{\arraystretch}{1.15}
    \setlength{\tabcolsep}{6pt}
    \begin{small}
    \begin{tabular}{ll}
        \toprule
        \textbf{Notation} & \textbf{Description} \\ 
        \midrule 
        \multicolumn{2}{l}{\textit{General Setup}} \\
        \midrule
        $\mathcal{G} = (\mathcal{V}, \mathcal{E})$ & Interaction graph \\
        $V = |\mathcal{V}|$ & Number of nodes \\
        $\{\mathcal{G}_t\}_{t=0}^{T+\Delta T}$ & Trajectory of graph states \\
        $T, \Delta T$ & Past length, prediction horizon \\
        $\mathbf{x}_t \in \mathbb{R}^{V \times 3}$ & Particle positions \\
        $\mathbf{v}_t \in \mathbb{R}^{V \times 3}$ & Particle velocities \\
        $\mathbf{z}_t \in \mathbb{R}^{V \times D_z}$ & Node attributes \\
        $\mathbf{X}_t \in \mathbb{R}^{V \times D}$ & Concatenated input, $D = 6 + D_z$ \\
        $\mathbf{A}_{\mathrm{adj}}$ & Adjacency matrix \\
        $\tilde{\mathbf{L}}$ & Normalized adjacency with self-loops \\
        $\mathcal{F}^{\dagger}, \mathcal{F}_{\theta}$ & True and learned solution operators \\
        $p_{\mathrm{data}}, \mathcal{L}$ & Data distribution, $L_2$ trajectory loss \\
        $\rho(\cdot), \|\cdot\|$ & Spectral radius, $L_2$ norm \\
        \midrule
        \multicolumn{2}{l}{\textit{State-Space Models (SSMs)}} \\
        \midrule
        $\mathbf{x}(t), \mathbf{h}(t), \mathbf{y}(t)$ & Continuous SSM input/state/output \\
        $\mathbf{x}[k], \mathbf{h}[k], \mathbf{y}[k]$ & Discrete SSM input/state/output \\
        $\mathbf{A}_{\mathrm{ssm}}, \mathbf{B}_{\mathrm{ssm}}, \mathbf{C}_{\mathrm{ssm}}$ & Continuous SSM matrices \\
        $\bar{\mathbf{A}}, \bar{\mathbf{B}}$ & Discretized SSM matrices (ZOH) \\
        $\Delta \in \mathbb{R}_+$ & Discretization step size \\
        $\boldsymbol{\phi}(\cdot)$ & Koopman observable map \\
        $\mathcal{K}, \mathbf{K}, \bar{\mathbf{K}}$ & Koopman operator, finite approx., discrete \\
        \midrule
        \multicolumn{2}{l}{\textit{\method-Specific Components}} \\
        \midrule
        $D, N$ & Feature dim., SSM state dim. \\
        $\mathbf{H}_t \in \mathbb{R}^{V \times D \times N}$ & Latent state (Koopman-inspired representation) \\
        $\mathbf{Y}_t, \mathbf{O}_t$ & Per-step readout, gated output \\
        $\mathbf{A} \in \mathbb{R}^{D \times N}$ & HiPPO temporal decay matrix \\
        $\mathbf{W}_s \in \mathbb{R}^{N \times N}$ & Spatial mixing matrix \\
        $\mathbf{B}_t, \mathbf{C}_t$ & Input-dependent SSM projections \\
        $\boldsymbol{\delta}_t, \mathbf{b}_\Delta, \Delta_t$ & Step-size logit, bias, and softplus output \\
        $\bar{\mathbf{A}}_t, \bar{\mathbf{B}}_t$ & Discretized propagator and input gain \\
        $\mathbf{K}_t$ & Effective graph-coupled propagator \\
        $\mathbf{Z}_t, \tilde{\mathbf{X}}_t$ & Output gate, auxiliary input \\
        $\mathbf{D}_{\mathrm{skip}}, \mathbf{W}_{\mathrm{out}}$ & Skip parameter, output projection \\
        $f_\theta$ & GNN generating $\mathbf{B}_t, \mathbf{C}_t, \boldsymbol{\delta}_t$ \\
        $\sigma$ & Activation function \\
        $\odot, \otimes$ & Hadamard, Kronecker product \\
        $L$ & Number of \method blocks \\
        \bottomrule
    \end{tabular}
    \end{small}
\end{table*}

\section{Additional Related Work}
\subsection{Graph Neural Networks}
Graph neural networks (GNNs) provide a natural framework for modeling articulated rigid body systems by representing bodies as nodes and interactions as edges~\citep{kipf2017semi,hamilton2017inductive,velivckovic2018graph}. Message-passing architectures have been successfully applied to learning physical interactions and dynamics~\citep{kipf2018neural,mrowca2018flexible}. However, their reliance on localized propagation limits the ability to capture long-range dependencies and multi-hop interactions, often leading to over-smoothing and vanishing gradients. Graph Transformers partially mitigate this issue via global attention, but incur quadratic complexity. Continuous-time formulations such as Graph Neural ODEs~\citep{fang2021spatial,luo2023hope} avoid discretization artifacts but require expensive numerical solvers and can be unstable. More general spatiotemporal graph models~\citep{zhang2018gaan,li2017diffusion,jin2023spatio} jointly model space and time, yet typically rely on autoregressive rollouts, leading to error accumulation and reduced robustness. Overall, most existing GNN-based approaches implicitly assume stationarity and short-term temporal dependence, limiting their effectiveness for long-horizon articulated dynamics.

\subsection{Neural Operators}
Neural operators learn mappings between function spaces and serve as data-driven alternatives to classical numerical solvers~\citep{kovachki2023neural}. Fourier Neural Operators (FNOs~\citep{li2020fourier,li2023fourier}) have demonstrated strong performance on PDE-driven systems by leveraging frequency-domain representations~\citep{bonev2023spherical}. Recent work integrates state-space models (SSMs) into operator architectures to better capture long-range temporal dependencies~\citep{gu2023mamba,tiwari2025latent}. However, Fourier-based formulations often assume stationarity and global spectral structure, limiting their ability to model time-varying and non-stationary dynamics. While extensions introduce temporal correlations within operator frameworks~\citep{xu2024equivariant}, spatial and temporal dynamics are still largely treated independently, restricting the ability to model tightly coupled spatiotemporal evolution in articulated systems.

\section{Preliminaries and Formal Proofs \label{app_sec:theory}}
\subsection{Preliminaries}

\textbf{Notation.} The graph encodes both static and dynamic particle attributes. Node features $\mathbf{H} \in \mathbb{R}^{V \times d}$ represent time-invariant properties such as atom types or charges. Time-varying geometric states are represented by
$\mathbf{Z}_t \in \mathbb{R}^{V \times m \times 3}$, typically instantiated as
$\mathbf{Z}_t = [\mathbf{x}_t, \mathbf{v}_t]$, where
$\mathbf{x}_t \in \mathbb{R}^{V \times 3}$ denotes particle positions and
$\mathbf{v}_t \in \mathbb{R}^{V \times 3}$ denotes velocities, yielding $m = 2$. An adjacency matrix specifies graph connectivity
$\mathbf{A}_{\mathrm{adj}} \in \mathbb{R}^{V \times V}$, constructed from domain-specific interactions
(e.g., chemical bonds) or distance-based neighborhoods (e.g., Euclidean distance).

\textbf{Newtonian--SSM Connection.} Let $\mathbf x,\mathbf v\in\mathbb R^q$ denote stacked positions and velocities, $\mathbf M\succ0$ a constant diagonal mass matrix, $\boldsymbol\Gamma\succeq0$ a constant damping matrix, and
$\mathbf f(\mathbf x,\mathbf v,t)$ the applied force. Newtonian dynamics obey the following:
\begin{align}
 \dot{\mathbf x}=\mathbf v,\qquad
 \mathbf M\dot{\mathbf v}=\mathbf f(\mathbf x,\mathbf v,t)-\boldsymbol\Gamma\mathbf v.
\end{align}
With $\mathbf h=[\mathbf x;\mathbf v]$, the equivalent first-order system is as follows:
\begin{align}
 \dot{\mathbf h}=
 \begin{bmatrix}\mathbf0&\mathbf I\\\mathbf0&-\mathbf M^{-1}\boldsymbol\Gamma\end{bmatrix}\mathbf h
 +\begin{bmatrix}\mathbf0\\\mathbf M^{-1}\end{bmatrix}\mathbf f(\mathbf x,\mathbf v,t).
\end{align}
The autonomous system is LTI when the force is linear in the state with constant coefficients and has no explicit time dependence. General forces give a nonlinear or time-dependent state-space system. In the conservative
case $\boldsymbol\Gamma=0$ and $\mathbf f=-\nabla U(\mathbf x)$, the canonical variables $\mathbf q=\mathbf x$ and $\mathbf p=\mathbf M\mathbf v$ satisfy
Hamilton's equations with $\mathcal H(\mathbf q,\mathbf p)=\tfrac12\mathbf p^\top\mathbf M^{-1}\mathbf p+U(\mathbf q)$. General damping and nonconservative forces do not have this standard conservative Hamiltonian interpretation.

\textbf{Training Objective.} In practice, we approximate the continuous objective through empirical risk minimization (ERM) over $P$ uniformly sampled time points $\{\tau_p\}_{p=1}^P \subset [0, \Delta T]$:
\begin{align}
    \label{eq:erm_objective}
    \min_{\theta} \; \mathbb{E} \left[\frac{1}{P} \sum_{p=1}^P \left\| \mathcal{F}_\theta(\mathcal{G}^{(t)})(\tau_p) - \mathcal{F}^\dagger(\mathcal{G}^{(t)})(\tau_p) \right\|_2 \right].
\end{align}
We focus on modeling geometric evolution: the time-varying descriptors $\mathbf{Z}_t$ capture the system dynamics, while the structural node features $\mathbf{H}$ remain time-invariant. \footnote{For simplicity, we omitted $\mathcal{G}^{(t)} \sim p_{\mathrm{data}}$ in ERM.}

\subsection{Mathematical Proofs}
In this section, we analyze the discrete latent recurrence. We use the vectorization and operators in Equation~\eqref{eq:correct_transition}. The graph is fixed during each analyzed recurrence. State Jacobians condition on the input sequence; they do not include derivatives through observation feedback. The statements apply separately to the forward and reverse scans after reindexing.

\begin{theorem}[Input-Conditioned Affine Latent Dynamics]
\label{thm:input_conditioned_koopman}
For each fixed input $\mathbf{X}_t$, the update
\begin{equation}
 \mathbf{H}_{t+1}=\bar{\mathbf{A}}_t\odot(\tilde{\mathbf{L}}\mathbf{H}_t\mathbf{W}_s)
 +\bar{\mathbf{B}}_t\odot\mathbf{X}_t
 \label{eq:koopman_graph_update}
\end{equation}
is affine in $\mathbf{H}_t$. Its homogeneous transition is
\begin{equation}
 \mathbf{K}_t=\operatorname{diag}(\operatorname{vec}\bar{\mathbf{A}}_t)
 (\tilde{\mathbf{L}}\otimes\mathbf{I}_D\otimes\mathbf{W}_s^\top).
 \label{eq:input_conditioned_koopman}
\end{equation}
Consequently,
\begin{equation}
 \mathbf{h}_{t+1}=\mathbf{K}_t\mathbf{h}_t+\mathbf{u}_t.
 \label{eq:koopman_linear_latent}
\end{equation}
This gives a Koopman-inspired structural interpretation, but does not by itself establish closure of a learned observable space under the physical dynamics.
\end{theorem}
\begin{proof}
With memory index fastest, vectorizing the graph and memory mixing gives
$\operatorname{vec}(\tilde{\mathbf{L}}\mathbf{H}_t\mathbf{W}_s)=\mathbf{S}\mathbf{h}_t$.
Coordinatewise multiplication by $\bar{\mathbf{A}}_t$ is multiplication by $\mathbf{D}_t$. The injection is independent of $\mathbf{h}_t$ for fixed input, proving the affine representation.
\end{proof}

Next, we show in the following lemma \ref{lemma:spectrum_and_neumann_series} that the normalized adjacency $\hat{\mathbf{L}}$ always has spectral radius $1$, which prevents the plain Neumann series from converging. Introducing a damping factor $\alpha \in (0,1)$ ensures convergence and yields the standard resolvent form.

\begin{lemma}[\textbf{Spectrum and Neumann series}] \label{lemma:spectrum_and_neumann_series}
Let $\mathcal{G}$ be an undirected graph with nonnegative edge weights and
\begin{align}
    \hat{\mathbf{L}} \;=\; \mathbf{D}^{-1/2}(\mathbf{A}_{\mathrm{adj}} + \mathbf{I})\mathbf{D}^{-1/2},\qquad 
    \mathbf{D}=\operatorname{Diag}( \mathbf{A}_{\mathrm{adj}}\mathbf{1}+\mathbf{1}).
\end{align}
Then $\sigma(\hat{\mathbf{L}})\subseteq[-1,1]$ and $\|\hat{\mathbf{L}}\|_2=\rho(\hat{\mathbf{L}})=1$, with eigenvalue $1$ having eigenvector $\mathbf{D}^{1/2}\mathbf{1}$. 
Hence $\sum_{t=0}^{\infty}\hat{\mathbf{L}}^t$ diverges, while for any $\alpha\in(0,1)$,
\begin{equation}
    \sum_{t=0}^{\infty}(\alpha \hat{\mathbf{L}})^t=(\mathbf{I}-\alpha \hat{\mathbf{L}})^{-1}
\quad\text{(converges in operator norm).}
\end{equation}
\end{lemma}

\begin{proof}
$\hat{\mathbf{L}}$ is symmetric, so its spectrum is real and $\|\hat{\mathbf{L}}\|_2=\rho(\hat{\mathbf{L}})$. 

Define the normalized Laplacian as
\begin{align}
    \mathcal{L} := \mathbf{I} - \hat{\mathbf{L}} 
    = \mathbf{D}^{-1/2}\bigl(\mathbf{D}-(\mathbf{A}_{\mathrm{adj}}+\mathbf{I})\bigr)\mathbf{D}^{-1/2},
\end{align}
which is symmetric positive semi-definite.  
The quadratic form identity gives
\begin{align}
    z^\top \mathcal{L} z 
    = \tfrac{1}{2}\!\sum_{i,j}(\mathbf{A}_{\mathrm{adj}}+\mathbf{I})_{ij}\bigl(y_i-y_j\bigr)^2,
    \quad y=\mathbf{D}^{-1/2}z,
\end{align}
hence $0 \leq \mu \leq 2$ for every eigenvalue $\mu$ of $\mathcal{L}$ (using $(a-b)^2 \leq 2(a^2+b^2)$).  
Thus every eigenvalue $\lambda$ of $\hat{\mathbf{L}} = \mathbf{I}-\mathcal{L}$ satisfies $\lambda \in [-1,1]$.
Moreover,
\begin{align}
    \hat{\mathbf{L}}\,\mathbf{D}^{1/2}\mathbf{1}
    = \mathbf{D}^{-1/2}(\mathbf{A}_{\mathrm{adj}}+\mathbf{I})\mathbf{1}
    = \mathbf{D}^{-1/2}\mathbf{D}\mathbf{1}
    = \mathbf{D}^{1/2}\mathbf{1},
\end{align}
so $1 \in \sigma(\hat{\mathbf{L}})$ with eigenvector $\mathbf{D}^{1/2}\mathbf{1}$.  
This implies $\rho(\hat{\mathbf{L}})=1$ and $\|\hat{\mathbf{L}}\|_2=1$.  
Hence $\hat{\mathbf{L}}^t(\mathbf{D}^{1/2}\mathbf{1}) = \mathbf{D}^{1/2}\mathbf{1}$ for all $t$, so 
$\sum_{t=0}^{T}\hat{\mathbf{L}}^t$ diverges (unbounded on that vector).

For $\alpha \in (0,1)$, $\|\alpha \hat{\mathbf{L}}\|_2 = \alpha < 1$, so the Neumann series converges and sums to
\begin{equation}
    \sum_{t=0}^{\infty} (\alpha \hat{\mathbf{L}})^t = (\mathbf{I} - \alpha \hat{\mathbf{L}})^{-1}.
\end{equation}
\end{proof}

Then we show the following proposition, which establishes that the proposed state-space update rule is permutation equivariant, meaning that relabeling the nodes of the input graph with any permutation matrix $\mathbf{P}$ leads to correspondingly permuted outputs. The proof relies on the fact that the graph operator $\hat{\mathbf{L}}$ conjugates naturally under permutations, while the learnable parameters $\mathbf{W}, \mathbf{B}, \mathbf{C}$ act only on feature dimensions and thus commute with node permutations. This ensures that the model’s predictions are independent of node ordering, 
a fundamental property for GNNs.

\begin{proposition}[\textbf{Permutation Equivariance}]
\label{prop:permutation_equivariance}
Let $\mathcal{G}=(\mathcal{V},\mathcal{E})$ be a graph with $|\mathcal{V}|=V$ nodes. Consider the \method update
\begin{equation}
    \mathbf{H}_{t+1}
    = \bar{\mathbf{A}}_t \odot \bigl( \tilde{\mathbf{L}} \, \mathbf{H}_t \, \mathbf{W}_s \bigr)
    + \bar{\mathbf{B}}_t \odot \mathbf{X}_t,
    \qquad
    \mathbf{Y}_t = \mathbf{H}_t \cdot \mathbf{C}_t,
\end{equation}
where $\mathbf{H}_t \in \mathbb{R}^{V \times D \times N}$, $\mathbf{X}_t \in \mathbb{R}^{V \times D}$,
$\tilde{\mathbf{L}} \in \mathbb{R}^{V \times V}$ is the normalized adjacency matrix,
and $\mathbf{W}_s \in \mathbb{R}^{N \times N}$ acts on the state dimension.
The discretized parameters are
$\bar{\mathbf{A}}_t = \exp(\Delta_t \odot \mathbf{A})$ and
$\bar{\mathbf{B}}_t = \Delta_t \odot \mathbf{B}_t$,
where $\mathbf{A} \in \mathbb{R}^{D \times N}$ is shared across nodes, and
$\Delta_t$, $\mathbf{B}_t$, and $\mathbf{C}_t$ are generated by an equivariant graph convolution
$f_\theta(\mathbf{X}_t, \mathcal{G})$.

Then, for any permutation matrix $\mathbf{P} \in \{0,1\}^{V \times V}$, if the inputs transform as
\begin{align}
    \tilde{\mathbf{L}} \mapsto \mathbf{P}\tilde{\mathbf{L}}\mathbf{P}^\top,
    \qquad
    \mathbf{X}_t \mapsto \mathbf{P}\mathbf{X}_t,
    \qquad
    \mathbf{H}_0 \mapsto \mathbf{P}\mathbf{H}_0,
\end{align}
The resulting states and outputs satisfy, for all $t \ge 0$,
\begin{align}
    \mathbf{H}_t \mapsto \mathbf{P}\mathbf{H}_t,
    \qquad
    \mathbf{Y}_t \mapsto \mathbf{P}\mathbf{Y}_t.
\end{align}
\end{proposition}

\begin{proof}
We proceed by induction on $t$.
The base case holds by assumption: $\mathbf{H}_0 \mapsto \mathbf{P}\mathbf{H}_0$.

Assume $\mathbf{H}_t \mapsto \mathbf{P}\mathbf{H}_t$.
Since $f_\theta$ is permutation equivariant, its outputs transform as
\begin{align}
    \Delta_t \mapsto \mathbf{P}\Delta_t,
    \qquad
    \mathbf{B}_t \mapsto \mathbf{P}\mathbf{B}_t,
    \qquad
    \mathbf{C}_t \mapsto \mathbf{P}\mathbf{C}_t.
\end{align}
Because $\mathbf{A}$ is node-independent and broadcast across the node dimension,
\begin{align}
    \bar{\mathbf{A}}_t
    = \exp(\Delta_t \odot \mathbf{A})
    \mapsto
    \exp((\mathbf{P}\Delta_t) \odot \mathbf{A})
    = \mathbf{P}\bar{\mathbf{A}}_t,
\end{align}
and similarly $\bar{\mathbf{B}}_t \mapsto \mathbf{P}\bar{\mathbf{B}}_t$.

For the spatial propagation term, we have
\begin{align}
    (\mathbf{P}\tilde{\mathbf{L}}\mathbf{P}^\top)(\mathbf{P}\mathbf{H}_t)\mathbf{W}_s
    = \mathbf{P}\tilde{\mathbf{L}}(\mathbf{P}^\top\mathbf{P})\mathbf{H}_t\mathbf{W}_s
    = \mathbf{P}(\tilde{\mathbf{L}}\mathbf{H}_t\mathbf{W}_s),
\end{align}
where $\mathbf{P}^\top\mathbf{P}=\mathbf{I}$ and $\mathbf{W}_s$ commutes with node permutations
since it acts only on the state dimension.

Since both operands of each Hadamard product transform equivariantly, the full update satisfies $\mathbf{H}_{t+1} \mapsto \mathbf{P}\mathbf{H}_{t+1}$.
Finally, the readout $\mathbf{Y}_t = \mathbf{H}_t \cdot \mathbf{C}_t$ also obeys
$\mathbf{Y}_t \mapsto \mathbf{P}\mathbf{Y}_t$.
\end{proof}

Proposition~\ref{prop:walk_accumulation} describes graph-path propagation while retaining the interleaved gates and memory mixing.

\begin{proposition}[Information Flow via Graph Walk Accumulation] \label{prop:walk_accumulation}
With zero input injection, the discrete recurrence satisfies
\begin{equation}
 \mathbf{h}_{t+k}=\mathbf{K}_{t+k-1}\cdots\mathbf{K}_t\mathbf{h}_t.
\end{equation}
If the transition is time-independent, this is $\mathbf{K}^k\mathbf{h}_t$.
For $i_k=i$, $n_k=n$, its tensor entries are
\begin{equation}
 H_{t+k}(i,d,n)=
 \sum_{i_0,\ldots,i_{k-1}}\sum_{n_0,\ldots,n_{k-1}}
 H_t(i_0,d,n_0)
 \prod_{\ell=1}^{k}
 \left[\bar A_{t+\ell-1}(i_\ell,d,n_\ell)
 \tilde L_{i_\ell i_{\ell-1}}(W_s)_{n_{\ell-1}n_\ell}\right].
\end{equation}
Thus, contributions are supported on length-$k$ graph walks, with memory-path and gate-dependent weights. Self-loops permit paths between nodes at distance at most $k$; cancellations or zero memory weights can still make an influence zero.
\end{proposition}
\begin{proof}
The one-step coefficient from $(j,d,m)$ to $(i,d,n)$ is
$\bar A_t(i,d,n)\tilde L_{ij}(W_s)_{mn}$.
Repeated substitution multiplies these coefficients along graph and memory paths and sums over intermediate indices, giving both formulas. Gates cannot in general be moved through $\mathbf{S}$: the separated form $\mathbf{K}^k=\mathbf{D}^k\mathbf{S}^k$ requires $\mathbf{D}\mathbf{S}=\mathbf{S}\mathbf{D}$.
\end{proof}

\begin{proposition}[Spatiotemporal Jacobian]
\label{prop:spatiotemporal_jacobian}
Suppose $\bar{\mathbf{A}},\bar{\mathbf{B}},\mathbf{C}$ are fixed independently of inputs. Define $\mathbf{x}_t=\operatorname{vec}(\mathbf{X}_t)$, the injection matrix $\mathbf{E}$ by
$\mathbf{E}\mathbf{x}=\operatorname{vec}(\bar{\mathbf{B}}\odot\mathbf{X})$, and the readout matrix $\mathbf{R}$ by
$\mathbf{R}\mathbf{h}=\operatorname{vec}(\mathbf{H}\cdot\mathbf{C})$.
For $t>s$, with initial state independent of the inputs,
\begin{equation}
 \frac{\partial\mathbf{h}_t}{\partial\mathbf{x}_s}=\mathbf{K}^{t-s-1}\mathbf{E},
 \qquad
 \frac{\partial\operatorname{vec}\mathbf{Y}_t}{\partial\mathbf{x}_s}
 =\mathbf{R}\mathbf{K}^{t-s-1}\mathbf{E}.
\end{equation}
The $(i,j)$ node blocks give the corresponding nodewise Jacobians.
\end{proposition}
\begin{proof}
For $m=t-s$, finite unrolling gives
$\mathbf{h}_t=\mathbf{K}^{m}\mathbf{h}_s+
\sum_{j=0}^{m-1}\mathbf{K}^{j}\mathbf{E}\mathbf{x}_{t-1-j}$.
Only the summand with $j=m-1$ depends on $\mathbf{x}_s$.
Differentiate and apply the readout. For input-dependent selectors, additional selector derivatives are required; the displayed formulas are for the fixed-coefficient case only.
\end{proof}

\begin{proposition}[One-step Jacobian of the State Update]
\label{prop:jacobian_one_step}
With the input fixed, $\partial\mathbf{h}_{t+1}/\partial\mathbf{h}_t=\mathbf{K}_t$.
Its node blocks are
\begin{equation}
 (\mathbf{K}_t)_{ij}=\tilde L_{ij}\,
 \operatorname{blockdiag}_{d=1}^{D}
 \left[\operatorname{diag}(\bar A_t(i,d,:))\mathbf{W}_s^\top\right].
\end{equation}
\end{proposition}
\begin{proof}
Differentiate the affine recurrence with the input, and hence its selectors, fixed. The entry from $(j,e,m)$ to $(i,d,n)$ is
$\delta_{de}\bar A_t(i,d,n)\tilde L_{ij}(W_s)_{mn}$, giving the stated blocks.
\end{proof}

\begin{corollary}[Operator Norm and Spectral Bound]
In the Euclidean induced operator norm,
\begin{equation}
 \rho(\mathbf{K}_t)\le\|\mathbf{K}_t\|_2
 \le\|\bar{\mathbf{A}}_t\|_{\max}\|\tilde{\mathbf{L}}\|_2\|\mathbf{W}_s\|_2.
\end{equation}
A uniform upper bound strictly below one on the right-hand side is sufficient for uniform conditional state contraction.
\end{corollary}

\begin{lemma}[Uniform Bound on the Graph-Coupled Transition]
\label{lemma:kt_bound}
Let $\mathbf{K}_t=\mathbf{D}_t\mathbf{S}$ as in Equation~\eqref{eq:correct_transition}. Assume
$\|\tilde{\mathbf{L}}\|_2\le1$, $\|\mathbf{W}_s\|_2\le c_s$, and
$\sup_t\|\bar{\mathbf{A}}_t\|_{\max}\le c_A$.
Then $\|\mathbf{K}_t\|_2\le c_Ac_s$ for all $t$; if $c_Ac_s<1$, the transition is uniformly contractive.
\end{lemma}
\begin{proof}
Since $\mathbf{D}_t$ is diagonal,
$\|\mathbf{D}_t\|_2=\|\bar{\mathbf{A}}_t\|_{\max}$.
Submultiplicativity and the Kronecker-product norm identity give
\begin{equation}
 \|\mathbf{K}_t\|_2\le\|\mathbf{D}_t\|_2\|\mathbf{S}\|_2
 =\|\bar{\mathbf{A}}_t\|_{\max}\|\tilde{\mathbf{L}}\|_2\|\mathbf{W}_s\|_2
 \le c_Ac_s.
\end{equation}
No general Hadamard-product spectral-norm inequality is needed.
\end{proof}
\begin{remark}
Negative entries of $\mathbf{A}$ and positive $\Delta_t$ give entrywise gains in $(0,1)$, but do not alone imply a uniform contraction margin: gains can approach one and $\|\mathbf{W}_s\|_2$ can exceed one. The contraction results are conditional on the stated bound, not guaranteed by HiPPO initialization or for the full residual architecture.
\end{remark}

\begin{proposition}[Multi-step Jacobian of the State Update]
\label{prop:jacobian_multi_step}
For a fixed input sequence and $t>s$,
\begin{equation}
 \frac{\partial\mathbf{h}_t}{\partial\mathbf{h}_s}
 =\mathbf{K}_{t-1}\cdots\mathbf{K}_s.
\end{equation}
If $\sup_r\|\mathbf{K}_r\|_2\le q<1$, its norm is at most $q^{t-s}$.
\end{proposition}
\begin{proof}
Apply the chain rule to the conditional one-step Jacobians and then operator-norm submultiplicativity. For constant coefficients, the product is $\mathbf{K}^{t-s}$; no factor commutation is assumed.
\end{proof}

The next proposition states input adaptation without claiming that it implies statistical non-stationarity or infinite memory.
\begin{proposition}[Input-Dependent Temporal Weights]
\label{prop:non_stationary_temporal_dynamics}
Define $\mathbf{E}_t$ and $\mathbf{R}_t$ as the input-conditioned injection and readout matrices. For an initial time $t_0$,
\begin{equation}
 \operatorname{vec}\mathbf{Y}_t=
 \mathbf{R}_t\mathbf{P}_{t,t-t_0}\mathbf{h}_{t_0}
 +\sum_{j=0}^{t-t_0-1}\mathbf{R}_t\mathbf{P}_{t,j}
 \mathbf{E}_{t-1-j}\mathbf{x}_{t-1-j},
 \qquad \mathbf{P}_{t,j}=\mathbf{K}_{t-1}\cdots\mathbf{K}_{t-j},
\end{equation}
where $\mathbf{P}_{t,0}=\mathbf{I}$. These weights are functions of the input trajectory and may vary with it.
\end{proposition}
\begin{proof}
Repeated substitution in $\mathbf{h}_{t+1}=\mathbf{K}_t\mathbf{h}_t+\mathbf{E}_t\mathbf{x}_t$, followed by the readout, gives the formula. Selectors shared across time yield the same update rule at each step. Thus, input adaptation alone does not imply failure of time-shift equivariance, statistical non-stationarity, or impossibility of finite-window representations. Nilpotent transitions, for example, can yield finite memory.
\end{proof}

Finally, we give the moving-average expansion of the adaptive graph state-space recurrence. This provides an ARMA-like interpretation, without identifying its node-dependent operators with classical fixed polynomial graph filters or treating deterministic inputs as stochastic innovations.

\begin{theorem}[Adaptive Graph Moving-Average Representation]
\label{theorem:graph_arma}
For the discrete update
\begin{equation}
 \mathbf{H}_{t+1}=\bar{\mathbf{A}}_t\odot(\tilde{\mathbf{L}}\mathbf{H}_t\mathbf{W}_s)
 +\bar{\mathbf{B}}_t\odot\mathbf{X}_t,
 \label{eq:gramo_discrete_app}
\end{equation}
use $\mathbf{h}_t$, $\mathbf{u}_t$, and $\mathbf{K}_t$ from Equation~\eqref{eq:correct_transition}, and define
$\mathbf{P}_{t,j}=\mathbf{K}_{t-1}\cdots\mathbf{K}_{t-j}$, with $\mathbf{P}_{t,0}=\mathbf{I}$.
Assume $\sup_{t\in\mathbb{Z}}\|\mathbf{K}_t\|_2\le q<1$ and
$M=\sup_{t\in\mathbb{Z}}\|\mathbf{u}_t\|_2<\infty$.
Then there is a unique bounded entire solution, given by
\begin{equation}
 \mathbf{h}_t=\sum_{j=0}^{\infty}\mathbf{P}_{t,j}\mathbf{u}_{t-1-j}.
 \label{eq:gramo_ma_app}
\end{equation}
For any finite initial time $t_0$, irrespective of contraction, the exact formula is
\begin{equation}
 \mathbf{h}_t=\mathbf{P}_{t,t-t_0}\mathbf{h}_{t_0}
 +\sum_{j=0}^{t-t_0-1}\mathbf{P}_{t,j}\mathbf{u}_{t-1-j}.
 \label{eq:finite_unroll}
\end{equation}
For zero initialization, the boundary term is zero. The finite sum can be written as an infinite sum by zero-padding the input injection before $t_0$.
\end{theorem}
\begin{proof}
    Repeated substitution proves Equation~\eqref{eq:finite_unroll}. Submultiplicativity gives $\|\mathbf{P}_{t,j}\|_2\le q^j$, so
    $\sum_{j\ge0}\|\mathbf{P}_{t,j}\mathbf{u}_{t-1-j}\|_2\le M/(1-q)$.
    The series therefore defines a bounded solution, as direct substitution verifies.
    For any bounded entire solution with $\sup_r\|\mathbf{h}_r\|_2\le C$,
    the finite-unrolling boundary satisfies
    $\|\mathbf{P}_{t,m}\mathbf{h}_{t-m}\|_2\le q^m C\to0$.
    Taking $m\to\infty$ gives the series. The difference of two bounded entire solutions obeys the homogeneous recurrence and the same vanishing bound, proving uniqueness. Without a bounded prehistory or equivalent vanishing-boundary condition, the boundary term must be retained.
\end{proof}

\textbf{Interpretation.} The recurrence has an input-conditioned first-order autoregressive state transition and a graph-path-weighted moving-average expansion. Node-dependent gates need not commute with the graph operator, so these weights are not generally polynomials in $\tilde{\mathbf{L}}$. The theorem does not assert a classical finite-order stochastic ARMA representation.

\subsection{Computational Complexity}
Table~\ref{tab:complexity} shows the computational costs corresponding to the per-layer computational complexity of a forward pass. Classical GNNs such as GCN and EGNN incur linear spatial costs in $E$, while attention-based SE(3)-Transformers introduce quadratic dependence on $D$. Diffusion-based models additionally incur a factor proportional to the number of denoising steps. In contrast, \method combines sparse graph propagation with linear-time state-space updates, resulting in complexity that scales linearly with the trajectory length and graph size, while avoiding quadratic feature interactions.

\begin{table}[t]
    \centering
    \caption{\textbf{Computational Complexity Comparison} per layer. $V, E, T, D, K$ and $N$ denote the number of nodes, edges, time steps, hidden dimension, diffusion steps, and SSM state size, respectively.}
    \label{tab:complexity}
    \small
    \renewcommand{\arraystretch}{1.2}
    \setlength{\tabcolsep}{8pt}
    \begin{tabular}{l lll}
    \toprule
    \textbf{Method} & \textbf{Spatial Cost} & \textbf{Temporal Cost} & \textbf{Total Complexity} \\
    \midrule
    GCN & $\mathcal{O}(ED)$ & -- & $\mathcal{O}(ED)$ \\
    EGNN & $\mathcal{O}(ED)$ & $\mathcal{O}(TD)$ & $\mathcal{O}(T(ED + D))$ \\
    SE(3)-Tr. & $\mathcal{O}(ED^2)$ & $\mathcal{O}(TD^2)$ & $\mathcal{O}(T(ED^2))$ \\
    GeoTDM & $\mathcal{O}(ED)$ & $\mathcal{O}(K T E D)$ & $\mathcal{O}(K T E D)$ \\
    GraphMamba & -- & $\mathcal{O}(TVDN)$ & $\mathcal{O}(TVDN)$ \\
    \midrule
    \textbf{\method (Ours)} & $\mathcal{O}(TEN)$ & $\mathcal{O}(TVDN)$ & $\mathcal{O}(T(E + VD)N)$ \\
    \bottomrule
    \end{tabular}
\end{table}

\textbf{Efficiency.} \method is substantially more efficient than probabilistic trajectory models such as GeoTDM \cite{han2024geometric} and SVAE \cite{xu2022socialvae}, while achieving comparable efficiency to deterministic trajectory prediction baselines such as EqMotion~\cite{xu2023eqmotion}. Compared to simpler equivariant models like EGNN, \method incurs a modest increase in computational cost, reflecting the additional modeling capacity required for long-range and non-stationary dynamics.

\section{Benchmark Details}
This section provides comprehensive details on the benchmarks used in our evaluation, including dataset characteristics; training, validation, and test splits; and the graph construction process.

\textbf{N-Body Simulation.} The N-body dataset~\citep{kipf2018neural,brandstetter2021geometric,satorras2021n} consists of simulated trajectories of interacting particles under various physical forces, including charged-particle systems, spring dynamics, and gravitational interactions. Each system contains a small number of particles (typically $N=5$ for charged and spring systems, and $N=10$ for gravity) evolving according to known physical laws. Node features include particle positions and velocities, while interactions are defined by fully connected graphs, with edge attributes encoding relative positional information. For the charged system, particles carry charges randomly chosen from $\{+1,-1\}$ and interact via Coulomb forces. In the spring system, particles with random masses are connected with springs sampled with probability $0.5$, following Hooke’s law. In the gravity system, particles evolve under pairwise gravitational forces with random initial conditions. For all three settings, we use $3000/600/600$ trajectories for train/validation/test splits, respectively. Each trajectory consists of multiple time steps, where we condition on the first $10$ frames and predict the next $20$.

\textbf{Robotics Simulation.} The Robotics dataset~\citep{li2019learning} consists of simulated trajectories from three control environments involving rope manipulation and soft-body deformation, designed to evaluate the ability of learned models to capture contact-rich and deformable dynamics. Each environment contains a small number of articulated nodes that evolve under internal forces, external actuation, and (in one setting) fluid coupling. Node features include 2D positions and velocities, while interactions are defined by fully connected graphs over the constituent elements, with edge attributes encoding relative positional information.

For the \emph{Rope} environment, the top mass of a rope is actuated horizontally along a fixed-height trajectory, while the remaining masses evolve freely under internal spring-like forces and gravity. Each mass is treated as a node with 2D position and velocity, yielding a $4$-dimensional observation per node and a $4N$-dimensional state space for a rope of $N$ masses. In the \emph{Soft} environment, a soft robot composed of deformable quadrilateral blocks is actuated by embedded contractile elements, with one block anchored to the ground as a fixed boundary condition. Each quadrilateral is treated as a node, with the positions and velocities of its four corners forming a $16$-dimensional observation per node, yielding a $16N$-dimensional state space for a robot of $N$ quadrilaterals. The \emph{Swim} environment uses the same soft-robot construction as \emph{Soft}, but removes the ground anchor, allowing the robot to evolve under fluid-coupling forces; the observation structure is identical to \emph{Soft}. For all three settings, we use a $70/15/15$ train/validation/test split. Each trajectory consists of multiple time steps, where we condition on the first $10$ frames and predict the next $20$.

\textbf{CMU Motion Capture.} The CMU Motion Capture dataset \citep{cmu2003motion} provides 3D trajectories of human actions. We evaluate on two subsets—Subject \#35 (\texttt{Walk}) and Subject \#9 (\texttt{Run})—adopting the data splits and preprocessing of \citet{huang2022equivariant,han2022equivariant}.  
Subject \#35 uses 22/12/12 trajectories for train / val / test, and Subject \#9 uses a 5/4/2 split. Each sample is represented as a skeletal graph with 31 nodes (joint locations) and edges encoding the connecting bones. 

\textbf{Molecular Dynamics.} The MD17 dataset~\citep{chmiela2017machine} contains molecular-dynamics trajectories for eight small molecules. Node features comprise the atom type concatenated with the speed \(\|\mathbf{v}\|_2\). Unlike common practice, hydrogen atoms are not removed, leaving only heavy-atom dynamics to be modeled. For the graph topology, we augment the native molecular graph by adding 2-hop edges as in prior work~\citep{shi2021learning,xu2022geodiff}; edge attributes are formed by concatenating the hop type, the atomic types of the two endpoints, and the chemical bond type. For training, we explicitly used a 2:4:4 split for train/val/test.

\textbf{Protein.} We utilize the preprocessed AdK equilibrium trajectories provided by \citet{han2022equivariant}, which originate from the AdK dataset of \citet{seyler5108170molecular} and are made available via the MDAnalysis toolkit \citep{oliver_beckstein-proc-scipy-2016}. The simulations employ the CHARMM27 force field \citep{mackerell2000development} with explicit solvent and ions under NPT conditions at \(300\)~K and \(1\)~bar. Trajectory frames were recorded every \(240\)~ps, yielding a total simulated time of \(1.004\,\mu\mathrm{s}\). 

\section{Implementation Details}
\label{app_sec:implementation_details}
This section provides the implementation details of the proposed method to ensure clarity and reproducibility. We describe the model architecture, graph construction, input and output representations, training setup, hyperparameter choices, and evaluation protocol used in our experiments. Unless otherwise specified, the same implementation settings are used across datasets, with dataset-specific details reported separately.

\subsection{Training Details}
All models are trained using the Adam optimizer~\citep{kingma2014adam} along with the StepLR scheduler. Experiments are carried out on a Linux system running Ubuntu 20.04.3LTS, equipped with an Intel(R) Core(TM) i9-10900X CPU and a single NVIDIA RTX A6000 GPU with 48 GB memory.

\subsection{Hyperparameter Details}
Table~\ref{tab:hyperparameters} summarizes the hyperparameter settings used across the four primary benchmarks (N-body, Robotics, MoCap, MD17) and the auxiliary AdK protein dataset. While the overall architecture remains consistent, key training parameters such as batch size, learning rate, and number of layers are adjusted to account for differences in dataset scale, trajectory length, and dynamical complexity. For the N-body and Robotics benchmarks, moderate-depth models with batch sizes of $64$--$100$ are sufficient given the small per-trajectory node counts and well-defined interaction laws. For MoCap, smaller batches and lower learning rates accommodate the longer articulated trajectories and higher variance across motion types. MD17 uses the largest batch size and learning rate, reflecting the abundance of training samples and the need to model fast molecular vibrations. Smaller batches and conservative learning rates are reserved for AdK, where limited data requires stable optimization to prevent overfitting.

\begin{table}[t]
    \centering
    \caption{Hyperparameter configurations for each dataset.}
    \label{tab:hyperparameters}
    \begin{tabular}{l|ccccc}
    \toprule
    \textbf{Hyperparameter} & \textbf{N-body} & \textbf{Robotics} & \textbf{MoCap} & \textbf{MD17} & \textbf{AdK Protein} \\ 
    \midrule
    Epochs              & 1000              & 1000              & 1000              & 1000              & 100 \\
    Batch Size          & 100               & 64                & 12                & 100               & 4 \\
    Learning Rate       & $5 \times 10^{-4}$ & $1 \times 10^{-4}$ & $5 \times 10^{-4}$ & $5 \times 10^{-3}$ & $5 \times 10^{-5}$ \\
    Time Embedding      & 32                & 16                & 32                & 32                & 32 \\
    Weight Decay        & $1 \times 10^{-12}$ & $1 \times 10^{-12}$ & $1 \times 10^{-12}$ & $1 \times 10^{-12}$ & $1 \times 10^{-4}$ \\
    Number of Layers    & 4                 & 4                 & 6                 & 6                 & 4 \\
    Hidden Dimension    & 64                & 64                & 16                & 64                & 64 \\
    \bottomrule
    \end{tabular}
\end{table}

\textbf{Time Embedding.} To incorporate temporal information, \method augments the input features with explicit encoding of the time index for each structure in the trajectory. Specifically, we construct a set of sinusoidal functions at varying frequencies to generate fixed-time embeddings. For a timestep $\Delta t_i$, the embedding is given as follows:
\begin{equation}
\begin{aligned}
    \text{emb}_{2j}   &= \sin \!\left( \frac{i}{10000^{2j/d_{\text{emb}}}} \right), \\
    \text{emb}_{2j+1} &= \cos \!\left( \frac{i}{10000^{2j/d_{\text{emb}}}} \right),
\end{aligned}
\end{equation}
where $d_{\text{emb}}$ is the dimensionality of the time embedding space.

To ensure the model is aware of the timestep information required for future trajectory prediction, we incorporate \emph{temporal embeddings}. This practice is widely adopted across domains—for example, positional encodings in large language models~\citep{vaswani2017attention} and timestep embeddings in diffusion models~\citep{ho2020denoising}.  
\textbf{Bidirectional Temporal Modeling.} To capture both past and future context along the trajectory, we extend Equation~\eqref{eq:gramo_discrete} to a bidirectional formulation. For a trajectory $\{\mathbf{X}_{t}\}_{t=1}^{T}$, we apply the coupled update in both forward ($t = 1 \to T$) and backward ($t = T \to 1$) directions, maintaining separate latent states with independent parameters. The final output sums the contributions from both directions. This bidirectional processing is inspired by bidirectional integration schemes in dynamical systems, allowing the state to incorporate both historical context and future constraints, which is particularly beneficial for noisy or partially observed trajectories.

\subsection{Baselines Details}
For all baseline comparisons, we use the official code bases released by the respective authors and follow their recommended hyperparameter settings. The baselines span three categories: frame-to-frame predictors, deterministic trajectory models, and probabilistic trajectory models, with the latter two receiving primary emphasis. To ensure a fair comparison with our multi-step input formulation on the Motion Capture benchmark, all baselines are adapted to consume the full historical trajectory through a linear encoder, following~\citet{yuan2025non}. For hyperparameter tuning, we adopt the protocol of~\citet{han2022equivariant}: a random search over the number of layers $\{4, 6, 8\}$, hidden dimensions $\{16, 32, 64\}$, learning rates $\{5\!\times\!10^{-3},\, 1\!\times\!10^{-3},\, 5\!\times\!10^{-4},\, 1\!\times\!10^{-4}\}$, and batch sizes $\{32, 64, 128, 256\}$, retaining the configuration with the best validation performance. All reported results are averaged over five independent runs.

\subsection{\method Block: Algorithm}
Algorithm~\ref{alg:gramo_block} summarizes the computation performed by a single \method block. At each time step, an equivariant graph convolution generates input-dependent state-space parameters that adapt the temporal dynamics to the current graph state. The latent state is then propagated through a discretized state-space update coupled with graph-based spatial mixing, optionally in both forward and backward time directions. A gated output projection produces the final node representations, enabling expressive spatiotemporal modeling while preserving permutation equivariance; conditional recurrence stability requires the bounds stated in Appendix~\ref{app_sec:theory}.

\begin{algorithm}[h]
    \caption{\textbf{Graph Mamba Operator (\method) Block}}
    \label{alg:gramo_block}
    \begin{spacing}{1.1}
    \begin{algorithmic}[1]
    \REQUIRE Node features $\mathbf{X}\in\mathbb{R}^{BV\times T\times D}$,
        graph $\mathcal{G}=(\mathcal{V},\mathcal{E})$, and normalized
        adjacency $\tilde{\mathbf{L}}\in\mathbb{R}^{V\times V}$ per graph
    \ENSURE Block output $\mathbf{O}\in\mathbb{R}^{BV\times T\times D_{\mathrm{out}}}$

    \STATE Apply the following operations per graph, where $B$ is the
        batch size, $V$ the number of nodes, and $N$ the memory dimension.
    \STATE Compute shared gate and skip features:
        $(\mathbf{z},\tilde{\mathbf{X}})
        \leftarrow \mathrm{GCL}_{\mathrm{out}}(\mathbf{X},\mathcal{G})$.

    \FOR{each enabled direction $r\in\{\mathrm{fw},\mathrm{bw}\}$}
        \STATE $[\mathbf{u}^r,\mathbf{B}^r,\mathbf{C}^r,\boldsymbol{\delta}^r]
            \leftarrow \mathrm{Split}
            (\mathrm{GCL}_r(\mathbf{X},\mathcal{G}))$
        \STATE $\mathbf{A}^r \leftarrow -\exp(\mathbf{A}_{\log}^r)$,
            \quad $\Delta^r \leftarrow
            \mathrm{softplus}(\boldsymbol{\delta}^r+\mathbf{b}_\Delta^r)$
        \STATE $\tilde{\mathbf{u}}^r \leftarrow
            \mathrm{Act}(\mathrm{Conv1d}_r(\mathbf{u}^r))$
        \STATE $\bar{\mathbf{A}}_t^r \leftarrow
            \exp(\Delta_t^r\odot\mathbf{A}^r)$,
            \quad $\bar{\mathbf{B}}_t^r \leftarrow
            \Delta_t^r\odot\mathbf{B}_t^r$ for all $t$,
            broadcasting $\Delta_t^r$ along the memory dimension.
    \ENDFOR

    \STATE Initialize $\mathbf{H}_0^{\mathrm{fw}}
        \leftarrow \mathbf{0}\in\mathbb{R}^{V\times D\times N}$.
    \FOR{$t=1,\ldots,T$}
        \STATE $\mathbf{H}_t^{\mathrm{fw}} \leftarrow
            \bar{\mathbf{A}}_t^{\mathrm{fw}}\odot
            (\tilde{\mathbf{L}}\mathbf{H}_{t-1}^{\mathrm{fw}}
            \mathbf{W}_s^{\mathrm{fw}})
            +\bar{\mathbf{B}}_t^{\mathrm{fw}}\odot
            \tilde{\mathbf{u}}_t^{\mathrm{fw}}$
        \STATE $Y_{t,id} \leftarrow
            \sum_{n=1}^{N} H_{t,idn}^{\mathrm{fw}} C_{t,in}^{\mathrm{fw}}$
    \ENDFOR

    \IF{bidirectional}
        \STATE Initialize $\mathbf{H}_{T+1}^{\mathrm{bw}}\leftarrow\mathbf{0}$.
        \FOR{$t=T,\ldots,1$}
            \STATE $\mathbf{H}_t^{\mathrm{bw}} \leftarrow
                \bar{\mathbf{A}}_t^{\mathrm{bw}}\odot
                (\tilde{\mathbf{L}}\mathbf{H}_{t+1}^{\mathrm{bw}}
                \mathbf{W}_s^{\mathrm{bw}})
                +\bar{\mathbf{B}}_t^{\mathrm{bw}}\odot
                \tilde{\mathbf{u}}_t^{\mathrm{bw}}$
            \STATE $Y_{t,id} \leftarrow Y_{t,id}
                + \sum_{n=1}^{N} H_{t,idn}^{\mathrm{bw}} C_{t,in}^{\mathrm{bw}}$
        \ENDFOR
    \ENDIF

    \STATE $\mathbf{O}_t \leftarrow
        \mathrm{LayerNorm}\!\left(
        (\mathbf{Y}_t+\mathbf{D}_{\mathrm{skip}}\odot\tilde{\mathbf{X}}_t)
        \odot\sigma(\mathbf{z}_t)\right)\mathbf{W}_{\mathrm{out}}$
        for all $t$.
    \end{algorithmic}
    \end{spacing}
\end{algorithm}

\section{Additional Experimental Results}
In this section, we provide additional experiments on MD17 and protein dynamics, along with ablation studies of \method, to further evaluate the effectiveness of the proposed approach.

\subsection{Molecular Dynamics}

\textbf{Experimental Setup.} We evaluate \method on the MD17~\citep{chmiela2017machine} dataset, which comprises molecular dynamics trajectories from DFT simulations of eight molecular systems. The molecules range in size from 9 atoms (Ethanol and Malonaldehyde) to 21 atoms (Aspirin). Following established experimental setups~\citep{huang2022equivariant}, each molecular trajectory is randomly split into 5000 training, 1000 validation, and 1000 test samples, using uniform temporal sampling. In contrast to some previous studies~\citep{xu2023eqmotion}, hydrogen atoms are explicitly included in the molecular graphs. Since hydrogen atoms dominate high-frequency vibrational modes, their inclusion leads to a more challenging and realistic prediction setting. Atom-level features are encoded using one-hot representations of atomic numbers~\citep{schutt2021equivariant}. The molecular graph is constructed by connecting atoms within three hops in the chemical bond graph~\citep{shi2021learning}. For each trajectory, we use the past 10 frames as input and predict the next 20 frames.

\textbf{Results.} Tables~\ref{tab:md17_amse} and~\ref{tab:md17_fmse} report trajectory and final-state prediction errors on the eight MD17 molecular systems. Averaged across all systems, \method achieves the lowest error on both metrics: AMSE drops from $0.71$ (NS-GNN) to $0.58$, an $18.3\%$ reduction, and FMSE drops from $1.44$ (GeoTDM) to $1.01$, a $29.9\%$ reduction. The relative gain is larger on FMSE than on AMSE, indicating that \method's advantage compounds at longer prediction horizons where iterative-rollout baselines accumulate error. Per-molecule trends reinforce this picture: \method takes the best score on six of the eight molecules in both AMSE and FMSE — including the two largest molecules, Aspirin (21 atoms) and Naphthalene (18 atoms) — while underperforming on Ethanol and Malonaldehyde, the two smallest systems with simpler near-harmonic dynamics that benefit less from adaptive latent state evolution. The pattern is consistent with the design of \method: graph-coupled latent dynamics with input-dependent operators provide the most benefit on larger molecules where long-range memory and many-body coupling dominate, and yield smaller margins on small, well-structured systems where simpler equivariant baselines already capture most of the dynamics. Trajectory-level baselines such as GeoTDM, EqMotion, and NS-GNN consistently outperform frame-to-frame predictors (TFN, EGNN, SE(3)-Tr.), confirming that mitigating roll-out error accumulation is the dominant axis of improvement on this benchmark.

\begin{table*}[t]
    \centering
    \caption{\textbf{Evaluation on the MD17 dataset.} AMSE ($\times 10^{-2}$) for all eight molecules. \one{First} and \two{Second} denote the best and second-best results. Lower values denote better performance.}
    \label{tab:md17_amse}
    \small
    \resizebox{\textwidth}{!}{
    \begin{tabular}{lcccccccc|c}
    \toprule
    \textbf{Method} & \textbf{Aspirin} & \textbf{Benzene} & \textbf{Ethanol} & \textbf{Malonaldehyde} & \textbf{Naphthalene} & \textbf{Salicylic} & \textbf{Toluene} & \textbf{Uracil} & \textbf{Average} \\
    \midrule
    RF 
    & $4.82_{\pm 0.48}$ & $0.76_{\pm 0.08}$ & $7.35_{\pm 0.74}$ & $4.63_{\pm 0.46}$
    & $1.48_{\pm 0.15}$ & $3.58_{\pm 0.36}$ & $2.08_{\pm 0.21}$ & $3.00_{\pm 0.30}$ & $3.46$ \\
    TFN 
    & $0.93_{\pm 0.09}$ & \two{$0.03_{\pm 0.00}$} & $2.12_{\pm 0.21}$ & $1.78_{\pm 0.18}$
    & $0.27_{\pm 0.03}$ & $0.69_{\pm 0.07}$ & \two{$0.43_{\pm 0.04}$} & $0.43_{\pm 0.04}$ & $0.84$ \\
    SE(3)-Tr. 
    & $4.54_{\pm 0.45}$ & $0.04_{\pm 0.00}$ & $1.86_{\pm 0.19}$ & $2.40_{\pm 0.24}$
    & $0.25_{\pm 0.03}$ & $1.88_{\pm 0.19}$ & $0.61_{\pm 0.06}$ & $0.60_{\pm 0.06}$ & $1.52$ \\
    EGNN 
    & $3.74_{\pm 0.37}$ & $0.05_{\pm 0.00}$ & $3.77_{\pm 0.38}$ & $8.12_{\pm 0.81}$
    & $0.47_{\pm 0.05}$ & $1.33_{\pm 0.13}$ & $2.25_{\pm 0.23}$ & $1.25_{\pm 0.13}$ & $2.62$ \\
    EqMotion 
    & $1.80_{\pm 0.18}$ & $0.04_{\pm 0.00}$ & \one{$1.21_{\pm 0.12}$} & \two{$1.26_{\pm 0.13}$}
    & $0.28_{\pm 0.03}$ & $0.64_{\pm 0.06}$ & $0.49_{\pm 0.05}$ & \two{$0.41_{\pm 0.04}$} & $0.77$ \\
    SVAE 
    & $4.76_{\pm 0.48}$ & $0.68_{\pm 0.07}$ & $7.86_{\pm 0.79}$ & $4.32_{\pm 0.43}$
    & $0.81_{\pm 0.08}$ & $0.78_{\pm 0.08}$ & $1.10_{\pm 0.11}$ & $1.10_{\pm 0.11}$ & $2.68$ \\
    GraphMamba 
    & $2.73_{\pm 0.27}$ & $0.05_{\pm 0.00}$ & $3.15_{\pm 0.48}$ & $6.12_{\pm 0.71}$
    & $0.57_{\pm 0.06}$ & $1.23_{\pm 0.23}$ & $1.85_{\pm 0.13}$ & $1.34_{\pm 0.24}$ & $2.13$ \\
    GeoTDM 
    & \two{$0.86_{\pm 0.06}$} & $0.04_{\pm 0.00}$ & $1.69_{\pm 0.07}$ & $1.40_{\pm 0.06}$
    & \two{$0.21_{\pm 0.02}$} & \two{$0.36_{\pm 0.04}$} & $0.76_{\pm 0.04}$ & $0.59_{\pm 0.03}$ & $0.74$ \\
    NS-GNN 
    & $0.93_{\pm 0.02}$ & $0.04_{\pm 0.00}$ & \two{$1.32_{\pm 0.04}$} & \one{$1.08_{\pm 0.07}$} 
    & $0.36_{\pm 0.05}$ & $0.61_{\pm 0.03}$ & $0.91_{\pm 0.02}$ & $0.42_{\pm 0.01}$ & \two{$0.71$} \\
    \midrule
    \method 
    & \one{$0.51_{\pm 0.02}$} & \one{$0.02_{\pm 0.00}$} & $1.53_{\pm 0.02}$ & $1.76_{\pm 0.23}$ 
    & \one{$0.18_{\pm 0.00}$} & \one{$0.19_{\pm 0.02}$} & \one{$0.26_{\pm 0.05}$} & \one{$0.16_{\pm 0.07}$} & \one{$0.58$} \\
    \bottomrule
    \end{tabular}}
\end{table*}

\begin{table*}[t]
    \centering
    \caption{\textbf{Evaluation on the MD17 dataset.} FMSE ($\times 10^{-2}$) for all eight molecules. \one{First} and \two{Second} denote the best and second-best results. Lower values denote better performance.}
    \label{tab:md17_fmse}
    \small
    \resizebox{\textwidth}{!}{
    \begin{tabular}{lcccccccc|c}
    \toprule
    \textbf{Method} & \textbf{Aspirin} & \textbf{Benzene} & \textbf{Ethanol} & \textbf{Malonaldehyde} & \textbf{Naphthalene} & \textbf{Salicylic} & \textbf{Toluene} & \textbf{Uracil} & \textbf{Average} \\
    \midrule
    RF 
    & $10.26_{\pm 1.03}$ & $3.95_{\pm 0.40}$ & $13.92_{\pm 1.39}$ & $10.83_{\pm 1.08}$
    & $1.80_{\pm 0.18}$ & $6.18_{\pm 0.62}$ & $3.26_{\pm 0.33}$ & $3.89_{\pm 0.39}$ & $6.76$ \\
    TFN 
    & $3.77_{\pm 0.38}$ & $0.13_{\pm 0.01}$ & $9.00_{\pm 0.90}$ & $7.82_{\pm 0.78}$
    & $0.50_{\pm 0.05}$ & $2.61_{\pm 0.26}$ & $1.18_{\pm 0.12}$ & $1.33_{\pm 0.13}$ & $3.29$ \\
    SE(3)-Tr. 
    & $16.24_{\pm 1.62}$ & $0.16_{\pm 0.02}$ & $6.77_{\pm 0.68}$ & $10.93_{\pm 1.09}$
    & $0.56_{\pm 0.06}$ & $5.11_{\pm 0.51}$ & $1.78_{\pm 0.18}$ & $2.02_{\pm 0.20}$ & $5.45$ \\
    EGNN 
    & $16.70_{\pm 1.67}$ & $0.09_{\pm 0.01}$ & $8.44_{\pm 0.84}$ & $48.21_{\pm 4.82}$
    & $0.93_{\pm 0.09}$ & $6.36_{\pm 0.64}$ & $4.54_{\pm 0.45}$ & $4.18_{\pm 0.42}$ & $11.18$ \\
    EqMotion 
    & $3.18_{\pm 0.32}$ & $0.10_{\pm 0.01}$ & $3.20_{\pm 0.32}$ & $3.26_{\pm 0.33}$
    & $0.44_{\pm 0.04}$ & $1.20_{\pm 0.12}$ & $0.87_{\pm 0.09}$ & $0.71_{\pm 0.07}$ & $1.62$ \\
    SVAE 
    & $9.63_{\pm 0.96}$ & $0.93_{\pm 0.09}$ & $13.40_{\pm 1.34}$ & $9.71_{\pm 0.97}$
    & $0.96_{\pm 0.10}$ & $1.06_{\pm 0.11}$ & $1.54_{\pm 0.15}$ & $1.28_{\pm 0.13}$ & $4.81$ \\
    GraphMamba 
    & $13.60_{\pm 2.68}$ & \two{$0.08_{\pm 0.02}$} & $6.44_{\pm 0.64}$ & $38.11_{\pm 3.83}$
    & $0.92_{\pm 0.09}$ & $5.87_{\pm 0.74}$ & $5.54_{\pm 0.35}$ & $5.18_{\pm 0.63}$ & $9.47$ \\
    GeoTDM 
    & \two{$1.96_{\pm 0.20}$} & \two{$0.08_{\pm 0.01}$} & $3.29_{\pm 0.23}$ & \two{$2.63_{\pm 0.16}$}
    & \two{$0.40_{\pm 0.04}$} & \two{$0.76_{\pm 0.08}$} & \two{$1.77_{\pm 0.08}$} & $0.62_{\pm 0.05}$ & \two{$1.44$} \\
    NS-GNN 
    & $2.06_{\pm 0.16}$ & $0.09_{\pm 0.02}$ & \one{$2.98_{\pm 0.18}$} & \one{$2.46_{\pm 0.29}$} & $0.48_{\pm 0.07}$ & $0.84_{\pm 1.30}$ & $2.11_{\pm 1.30}$ & \two{$0.58_{\pm 0.03}$} & $1.45$ \\
    \midrule
    \method 
    & \one{$0.74_{\pm 0.13}$} & \one{$0.04_{\pm 0.01}$} & \two{$3.15_{\pm 0.07}$} & $2.79_{\pm 0.28}$ & \one{$0.27_{\pm 0.03}$} & \one{$0.33_{\pm 0.08}$} & \one{$0.51_{\pm 0.08}$} & \one{$0.27_{\pm 0.08}$} & \one{$1.01$} \\
    \bottomrule
    \end{tabular}}
\end{table*}

\subsection{Protein Dynamics}

\begin{table*}[h]
    \centering
    \caption{\label{tab:main_protein}\small \textbf{Evaluation on the Protein dataset.} FMSE (\(\times 10^{-2}\)) on the AdK equilibrium trajectory dataset. \one{First} and \two{Second} denotes best results. Lower values denote better performance.}
    \begin{tabular}{lcccccccccc}
    \toprule
    \textbf{Dataset} & \textbf{Linear} & \textbf{MPNN} & \textbf{RF} & \textbf{EGNN} & \textbf{EGHN} & \textbf{EGNO} & \textbf{EGHNO} & \textbf{\method}  \\
    \midrule
    AdK   & 2.890 &  2.322 & 2.846 & 2.735 & 2.034 & 2.231 & \two{1.80} & \one{1.74} \\
    \bottomrule
    \end{tabular}
\end{table*}

\textbf{Experimental Setup.} We further evaluate \method on the AdK equilibrium trajectory dataset~\citep{seyler5108170molecular}, accessed through the MDAnalysis toolkit~\citep{oliver_beckstein-proc-scipy-2016}. The dataset contains molecular dynamics trajectories of apo adenylate kinase. We follow the experimental protocol and data splits of \citet{han2022equivariant}. Unlike the previous tasks, this setting considers state-to-trajectory (\stot) prediction, in which the model predicts the next 4 frames from a single molecular state under uniform temporal discretization.

\textbf{Results.} Table~\ref{tab:main_protein} reports FMSE, where lower values indicate better performance. \method achieves lower error than the baselines, reducing FMSE from $1.80$ to $1.74$ in the \stot setting. This indicates that \method can model equilibrium molecular dynamics and predict short future trajectories from a single observed state. The improvement is modest, as expected for AdK: the trajectory is closer to equilibrium and exhibits weaker regime changes than those of the other benchmarks. As a result, the advantage of adaptive temporal modeling is less pronounced. Nevertheless, \method remains competitive, demonstrating that the proposed latent graph-SSM formulation generalizes effectively even in near-equilibrium dynamics.

\subsection{Temporal Extrapolation}
A core motivation for \method is that a learned latent operator should remain stable when rolled out beyond the training horizon, rather than degrade as soon as it leaves the supervised range. We test this directly by training models on short horizons and asking them to predict farther into the future than they were ever trained.

\textbf{Experimental Setup.} We compare \method against a Compositional-Koopman baseline on three N-body simulations: \emph{Charged Particles}, \emph{Spring Dynamics}, and \emph{Gravity System}. Both models receive $T_{\mathrm{in}} = 10$ input frames. We consider two training regimes that differ only in the supervised prediction horizon: $T_{\mathrm{out}} = 10$ (top row) and $T_{\mathrm{out}} = 20$ (bottom row). At test time, each model is rolled out autoregressively for 30 future steps. A single forward pass produces $T_{\mathrm{out}}$ predictions; the most recent $T_{\mathrm{in}}$ of those are then re-fed as the new input window, and the procedure repeats until 30 frames are accumulated. For the Koopman baseline, this means the latent operator $\mathbf{A}$ is re-fit on each call from the model's own (increasingly extrapolated) predictions, while \method's parameters never change. We report per-step MSE on $600$ held-out trajectories. The shaded green region marks the trained horizon, the shaded orange region marks the pure-extrapolation regime, and the dashed vertical line indicates the boundary.

\textbf{Results.} Within the trained horizon, both models degrade smoothly, but at the boundary, the Koopman baseline collapses sharply while \method continues to extrapolate gracefully. The gap widens monotonically as the rollout enters the extrapolation region, with \method outperforming Koopman by orders of magnitude across all datasets and training regimes. Doubling the supervised horizon reduces the extrapolation error for both methods, but \method retains a substantially lower error throughout, indicating that its learned state-space update generalizes beyond the supervised horizon far more reliably than a linear Koopman operator fit on a short window.

\begin{figure}[!htb]
    \centering
    \includegraphics[width=0.8\textwidth]{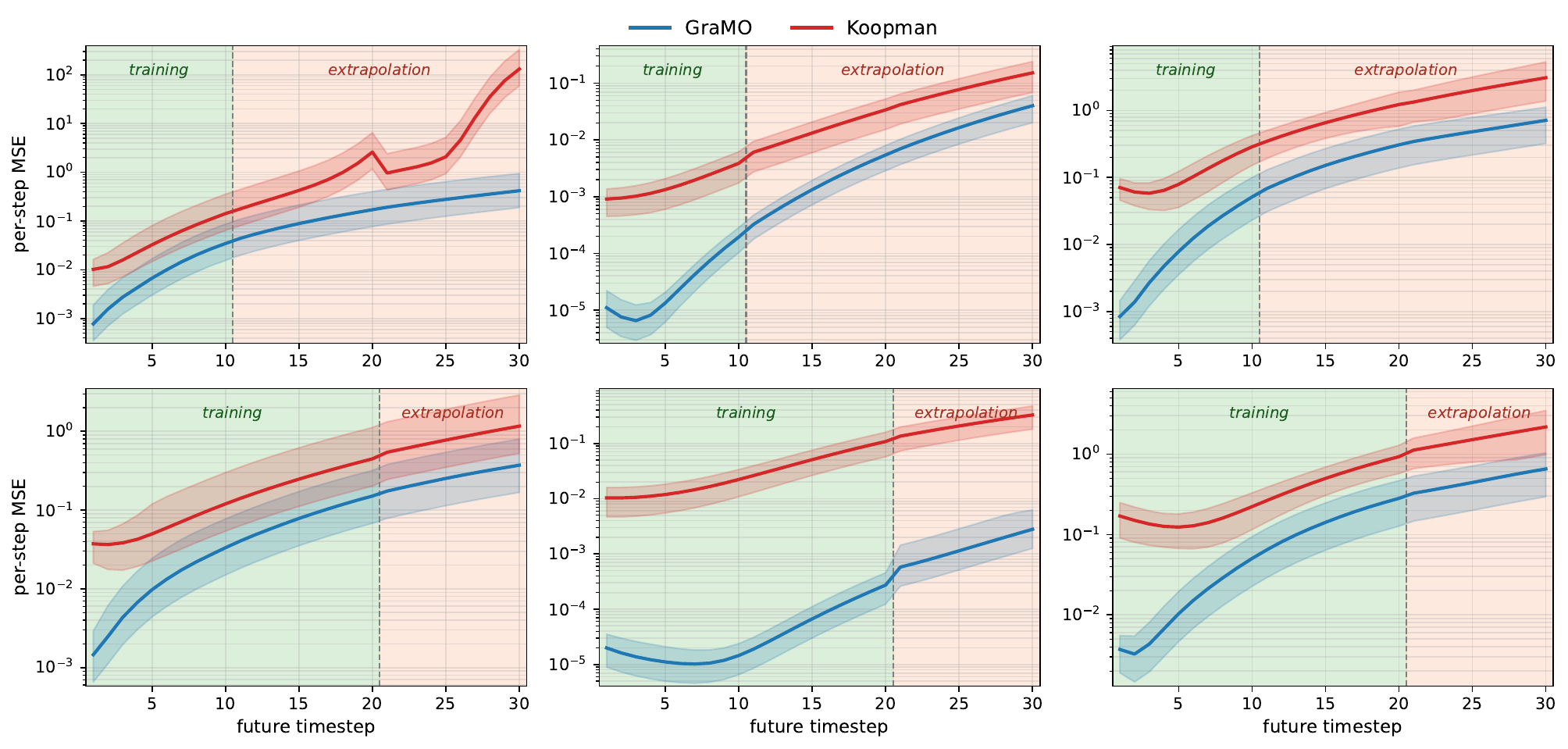}
    \caption{\label{fig:extrapolation} Visualization of Per-step trajectory error as a function of prediction horizon for \emph{Charged Particles} (left), \emph{Spring Dynamics} (middle), and \emph{Gravity System} (right). 
    The (\textbf{Top Row}) shows models trained on 10 timesteps and rolled out to 30; the (\textbf{Bottom Row}) shows models trained on 20 timesteps and rolled out to 30. The \textcolor{green!60!black}{shaded green} region marks the training horizon, and the \textcolor{orange!80!black}{shaded orange} region marks the pure extrapolation regime.}
\end{figure}

\subsection{Qualitative Visualizations}
In this subsection, we present visualizations of the dynamics predicted by \method. Results for particle simulations, \textsc{Mocap-Walk}, and \textsc{Mocap-Run} are shown in Figures~\ref{fig:mocap_walk_visualization}--\ref{fig:mocap_run_visualization}. As illustrated, \method not only produces accurate final-state predictions but also effectively captures the underlying dynamics by explicitly modeling both spatial and temporal correlations. Furthermore, Figure~\ref{fig:protein_visualization} provides a qualitative visualization of the trajectories predicted by \method. The figure illustrates how the model accurately reproduces the system's temporal evolution, closely aligning with the ground-truth dynamics and capturing both local and global structural changes. 

\begin{figure}[!htb]
    \centering
    \includegraphics[width=0.6\textwidth]{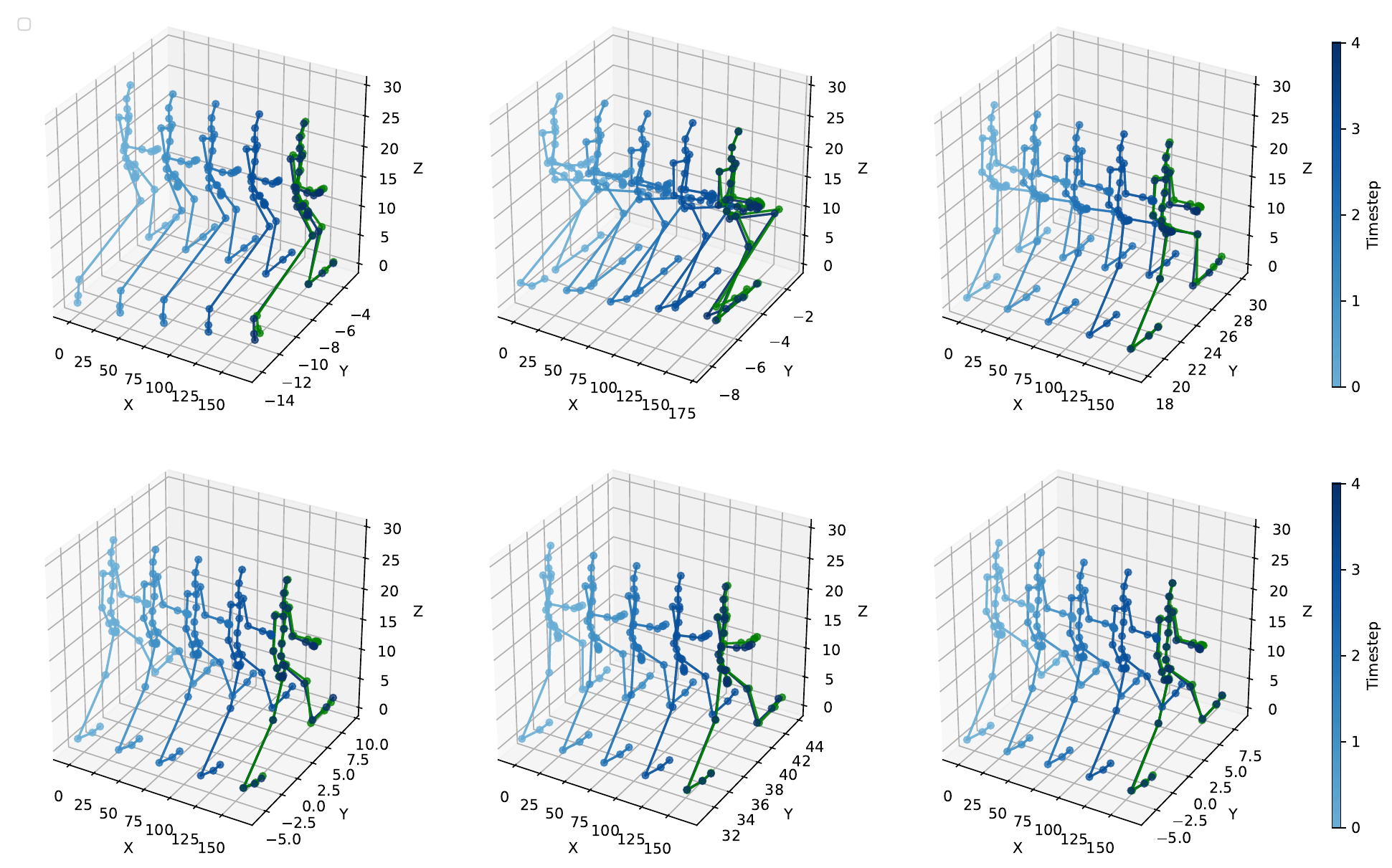}
    \caption{\label{fig:mocap_walk_visualization} Visualization of trajectories generated by \method with uniform discretization on the MoCap (Walk) dataset. Predicted trajectories are shown with timestep progression indicated by a \two{Blue} color gradient, while the ground truth final snapshot is marked in \one{Green}.}
\end{figure}

\begin{figure}[!htb]
    \centering
    \includegraphics[width=0.6\textwidth]{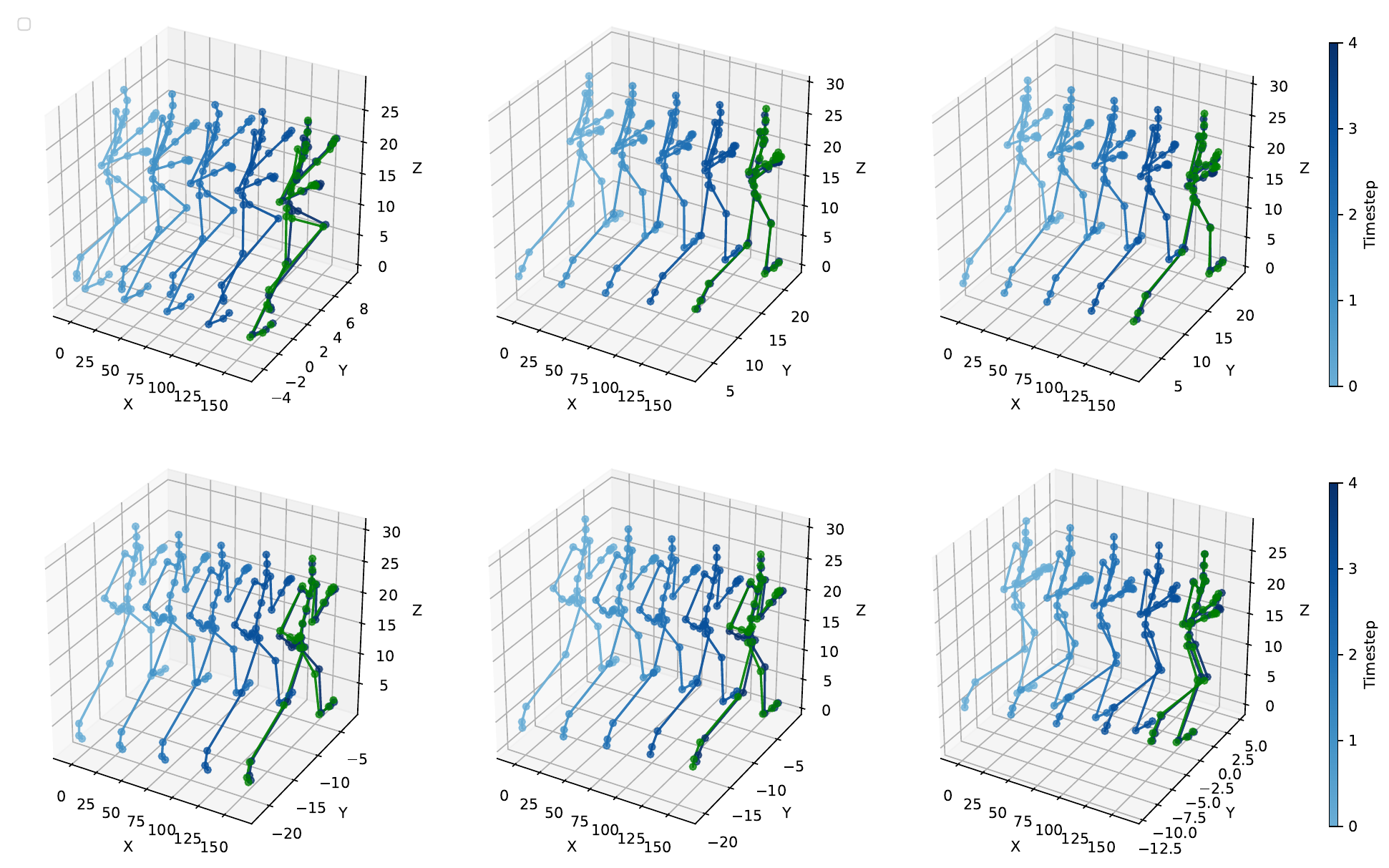}
    \caption{\label{fig:mocap_run_visualization} Visualization of trajectories generated by \method with uniform discretization on the MoCap (Run) dataset. Predicted trajectories are shown with timestep progression indicated by a \two{Blue} color gradient, while the ground truth final snapshot is marked in \one{Green}.}
\end{figure}

\begin{figure}[!htb] 
    \centering
    \includegraphics[width=0.7\textwidth]{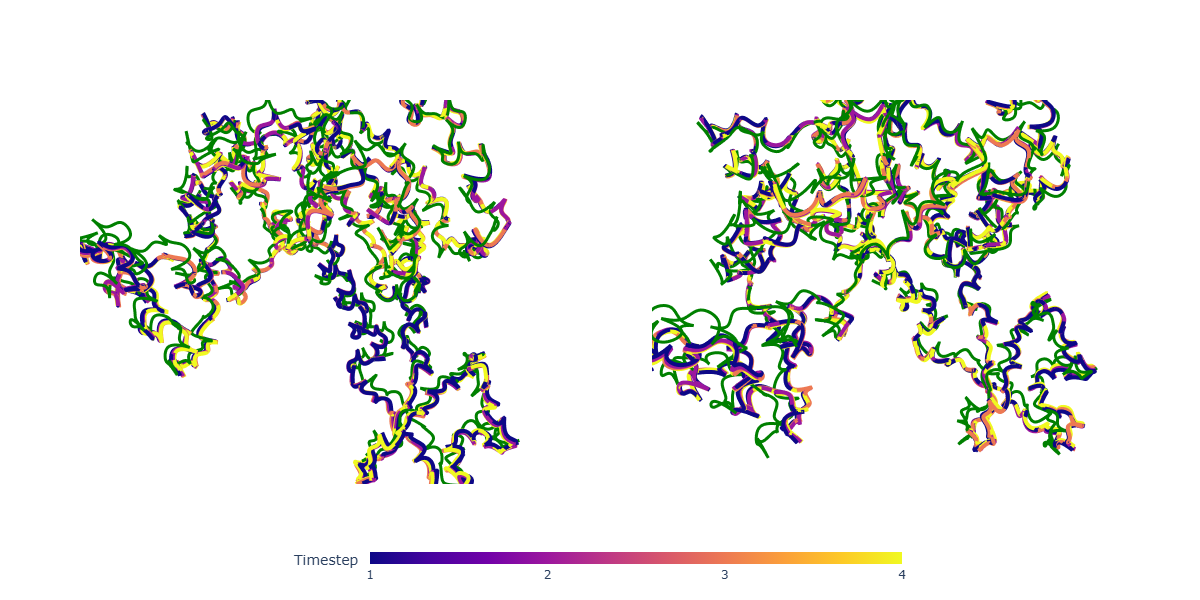}
    \caption{\label{fig:protein_visualization} Visualization of trajectories generated by \method with uniform discretization on AdK equilibrium trajectory dataset. Predicted trajectories are shown with timestep progression indicated by the colorbar above, while the ground-truth final snapshot is marked in \one{Green}.}
\end{figure}

\section{Limitations}
While \method demonstrates strong performance across a range of particle-based dynamical systems, several limitations remain. Although \method is designed to handle time-varying dynamics, its benefits may be less pronounced for systems that are close to stationary, where simpler models may suffice. In addition, the computational cost of \method increases with graph size due to repeated graph-coupled state updates, potentially limiting scalability for very large systems. Exploring graph sparsification and subgraph sampling strategies is a promising direction for improving scalability while preserving performance.

\section{Broader Impact}
This work introduces a graph-based neural framework for modeling interacting particle systems. By combining state-space modeling with graph-based message passing, the proposed approach aims to capture long-range interactions and time-varying dynamics more effectively than standard surrogate models. Improved learning-based simulators may support scientific and engineering workflows across physical simulation, robotics, materials modeling, and computational science.

The contributions of this work are primarily technical and methodological, focusing on fundamental challenges in learning dynamical systems on graphs. We do not anticipate direct negative ethical or societal impacts from this research. As with other general-purpose simulation tools, potential downstream effects depend on the application context, and we encourage responsible use consistent with established scientific and ethical guidelines.

\section{Use of Large Language Models (LLMs)}
Large language models (LLMs) were used as assistive tools to improve the clarity, grammar, and presentation of the manuscript. Specifically, LLMs were utilized for grammatical correction, rephrasing to enhance clarity, and refining the overall structure of the text. All LLM-assisted content was reviewed and edited by the authors to ensure it accurately reflects the authors’ intent and preserves the scientific content.



\end{document}